\documentclass{article}

\PassOptionsToPackage{numbers, compress}{natbib}

\usepackage[utf8]{inputenc} 
\usepackage[T1]{fontenc}    

\usepackage{lscape} 
\usepackage{booktabs}       
\usepackage{multirow,multicol}
\usepackage{amsfonts}       
\usepackage{nicefrac}       
\usepackage{microtype}      
\usepackage{amsthm}
\usepackage{amssymb}
\usepackage{amsmath}
\usepackage{dsfont}
\DeclareMathOperator*{\argmax}{arg\,max}
\DeclareMathOperator*{\argmin}{arg\,min}

\usepackage{mathrsfs}
\usepackage{graphicx}
\usepackage{algorithm}
\usepackage{algorithmic}
\usepackage{subcaption}
\usepackage{caption}

\DeclareCaptionLabelFormat{alglabel}{\bfseries Pseudocode}
\theoremstyle{plain}
\newtheorem{thm}{Theorem}
\newtheorem{lemma}[thm]{Lemma}

\usepackage[dvipsnames]{xcolor}
\definecolor{mydarkblue}{rgb}{0,0.08,0.45}
\usepackage[colorlinks=true,
            linkcolor=mydarkblue,
            citecolor=mydarkblue,
            filecolor=mydarkblue,
            urlcolor=mydarkblue]{hyperref}
\usepackage{cleveref}
\usepackage{url}
\usepackage{wrapfig}

 \usepackage[final]{neurips_2019}

\newcommand{\smoothAdv}{\textsc{SmoothAdv}}
\newcommand{\smoothAdvPGD}{$\smoothAdv_{\mathrm{PGD}}$}
\newcommand{\smoothAdvDDN}{$\smoothAdv_{\mathrm{DDN}}$}

\newcommand{\R}{\mathbb{R}}
\newcommand{\normal}{\mathcal{N}}

\newcommand{\eps}{\epsilon}
\newcommand{\CE}{\mathrm{CE}}
\DeclareMathOperator*{\E}{\mathbb{E}}

\title{
Provably Robust Deep Learning via Adversarially Trained Smoothed Classifiers}

\author{Hadi Salman$^\dagger$,
    Greg Yang$^\mathsection$,
    Jerry Li,
    \\
    \textbf{Pengchuan Zhang\footnotemark[1],$\ $
    Huan Zhang\footnotemark[1],$\ $
    Ilya Razenshteyn\footnotemark[1],$\ $
    S\'{e}bastien Bubeck\!
\thanks{Reverse alphabetical order.
$^\dagger$Work done as part of the \href{https://www.microsoft.com/en-us/research/academic-program/microsoft-ai-residency-program/}{Microsoft AI Residency Program}. $^\mathsection$Primary mentor.}}
\\
Microsoft Research AI\\
\texttt{
    \{hadi.salman,
    gregyang,
    jerrl,}\\
\texttt{
    penzhan,
    t-huzhan,
    ilyaraz,
    sebubeck\}@microsoft.com}
}

\begin{document}

\maketitle

\begin{abstract}
Recent works have shown the effectiveness of \textit{randomized smoothing} as a scalable technique for building neural network-based classifiers that are provably robust to $\ell_2$-norm adversarial perturbations. In this paper, we employ adversarial training to improve the performance of randomized smoothing. We design an adapted attack for smoothed classifiers, and we show how this attack can be used in an adversarial training setting to boost the \textit{provable} robustness of smoothed classifiers.
We demonstrate through extensive experimentation that our method consistently outperforms all existing provably $\ell_2$-robust classifiers by a significant margin on ImageNet and CIFAR-10, establishing the state-of-the-art for provable $\ell_2$-defenses. Moreover, we find that pre-training and semi-supervised learning boost adversarially trained smoothed classifiers even further.
Our code and trained models are available at \url{http://github.com/Hadisalman/smoothing-adversarial}.
\end{abstract}

\vskip -0.4cm
\section{Introduction}
\vskip -0.2cm
Neural networks have been very successful in tasks such as image classification and speech recognition, but have been shown to
be extremely brittle to small, adversarially-chosen perturbations of their inputs \citep{szegedy2013intriguing, goodfellow2014explaining}.
A classifier (e.g., a neural network), which correctly classifies an image $x$, can be fooled by an adversary to misclassify $x+\delta$ where $\delta$ is an adversarial perturbation so small that $x$ and $x + \delta$
are indistinguishable for the human eye.
Recently, many works have proposed heuristic defenses intended to train models robust to such adversarial perturbations.  However, most of these defenses were broken using more powerful adversaries  \citep{carlini2017adversarial, athalye2018obfuscated, uesato2018adversarial}.
This encouraged researchers to develop defenses that lead to \textit{certifiably robust} classifiers, i.e., whose predictions
for most of the test examples $x$ can be verified to be constant within a neighborhood of $x$ \citep{wong2018provable, raghunathan2018certified}.
Unfortunately, these techniques do not immediately scale
to large neural networks that are used in practice.

To mitigate this limitation of prior certifiable defenses, a number of papers \citep{lecuyer2018certified, li2018second, cohen2019certified} consider the \emph{randomized smoothing} approach, which transforms any classifier $f$ (e.g., a neural network) into a new \emph{smoothed} classifier $g$ that has certifiable $\ell_2$-norm robustness guarantees. This transformation works as follows.

Let $f$ be an arbitrary base classifier which maps inputs in $\mathbb{R}^d$ to classes in $\mathcal{Y}$.
Given an input $x$, the smoothed classifier $g(x)$ labels $x$ as having class $c$ which is the most likely to be returned by the base classifier $f$ when fed a noisy corruption $x + \delta$, where $\delta \sim \mathcal{N}(x, \sigma^2 I)$ is a vector sampled according to an isotropic Gaussian distribution.

As shown in~\cite{cohen2019certified}, one can derive certifiable robustness for such smoothed classifiers via the Neyman-Pearson lemma.
They demonstrate that for $\ell_2$ perturbations, randomized smoothing outperforms other certifiably robust classifiers that have been previously proposed.
It is scalable to networks with any architecture and size, which makes it suitable for building robust real-world neural networks.

\paragraph{Our contributions} In this paper, we employ adversarial training to substantially improve on the previous certified robustness results\footnote{Note that we do not provide a new certification method incorporating adversarial training; the improvements that we get are due to the higher quality of our base classifiers as a result of adversarial training.} of randomized smoothing \citep{lecuyer2018certified, li2018second, cohen2019certified}. 
We present, for the first time, a direct attack for smoothed classifiers. 
We then demonstrate how to use this attack to adversarially train smoothed models with not only boosted empirical robustness but also \textbf{substantially improved certifiable robustness} using the certification method of \citet{cohen2019certified}. 

We demonstrate that our method outperforms \textit{all} existing provably $\ell_2$-robust classifiers by a significant margin on ImageNet and CIFAR-10, establishing the state-of-the-art for provable $\ell_2$-defenses. For instance, our Resnet-50 ImageNet classifier achieves \textbf{$56\%$} provable top-1 accuracy (compared to the best previous provable accuracy of $49\%$) under adversarial perturbations with $\ell_2$ norm less than $127/255$.
Similarly, our Resnet-110 CIFAR-10 smoothed classifier achieves up to $16\%$ improvement over previous state-of-the-art,
and by combining our technique with pre-training \cite{hendrycks2019using} and semi-supervised learning \cite{carmon2019unlabeled}, we boost our results to up to $22\%$ improvement over previous state-of-the-art.
Our main results are reported in Tables~\ref{imagenet-certified-accuracy} and \ref{cifar-certified-accuracy} for ImageNet and CIFAR-10. See Tables~\ref{imagenet-certified-accuracy-with-clean-accuracy} and \ref{cifar-certified-accuracy-with-clean-accuracy} in Appendix~\ref{appendix:detialed_certification_results} for the standard accuracies corresponding to these results.

Finally, we provide an alternative, but more concise, proof of the tight robustness guarantee of  \citet{cohen2019certified}  by casting this as a \emph{nonlinear} Lipschitz property of the smoothed classifier. See appendix~\ref{appendix:alternative-proof} for the complete proof.

\begin{table}[t]
\caption{Certified top-1 accuracy of our best ImageNet classifiers at various $\ell_2$ radii.}
\label{imagenet-certified-accuracy}
\begin{center}
\begin{sc}
\begin{tabular}{l| c c c c c c c}
\toprule
$\ell_2$ Radius (ImageNet)& 0.5 & 1.0 & 1.5 & 2.0 & 2.5 & 3.0 & 3.5\\
\midrule
\citet{cohen2019certified} (\%) & 49 & 37 & 29 & 19 & 15 & 12 &9\\
Ours (\%)& \textbf{56} & \textbf{45} & \textbf{38} & \textbf{28} & \textbf{26} & \textbf{20} & \textbf{17}\\
\bottomrule
\end{tabular}
\end{sc}
\end{center}
\end{table}

\begin{table}[t]
\caption{Certified top-1 accuracy of our best CIFAR-10 classifiers at various $\ell_2$ radii.}
\label{cifar-certified-accuracy}
\begin{center}
\begin{sc}
\begin{tabular}{l | c c c c c c c c c}
\toprule
$\ell_2$ Radius (CIFAR-10)& $0.25$& $0.5$& $0.75$& $1.0$& $1.25$& $1.5$& $1.75$& $2.0$& $2.25$\\
\midrule
\citet{cohen2019certified} (\%) &61 & 43 & 32 & 22 & 17 & 13 & 10 & 7 & 4\\
Ours (\%)& 73 & 58 & 48 & 38 & 33 & 29 & 24 & 18 & 16\\
+ Pre-training (\%) & 80 & 62 & \textbf{52} & 38 & \textbf{34} & \textbf{30} & \textbf{25} & \textbf{19} & 16\\
+ Semi-supervision (\%) & 80 & \textbf{63} & \textbf{52} & \textbf{40} & \textbf{34} & 29 & \textbf{25} & \textbf{19} & \textbf{17}\\
+ Both(\%)  & \textbf{81} & \textbf{63} & \textbf{52} & 37 & 33 & 29 & \textbf{25} & 18 & 16\\
\bottomrule
\end{tabular}
\end{sc}
\end{center}
\end{table}

\vskip -0.2cm
\section{Our techniques} 
Here we describe our techniques for adversarial attacks and training on smoothed classifiers.
We first require some background on randomized smoothing classifiers.
For a more detailed description of randomized smoothing, see~\citet{cohen2019certified}.
\subsection{Background on randomized smoothing}\label{sec:randomized-smoothing}
\vskip -0.2cm
Consider a classifier $f$ from $\mathbb{R}^d$ to classes $\mathcal{Y}$. Randomized smoothing is a method that constructs a new, \textit{smoothed} classifier $g$ from the \textit{base} classifier $f$.
The smoothed classifier $g$ assigns to a query point $x$ the class which is most likely to be returned by the base classifier $f$ under isotropic Gaussian noise perturbation of $x$, i.e.,
\begin{align}
g(x) &= \argmax_{c \in \mathcal{Y}} \; \mathbb{P}(f(x+\delta) = c) \label{eq:smoothed-hard} \quad \text{where} \; \delta \sim \mathcal{N}(0, \sigma^2 I) \; .
\end{align}
The noise level $\sigma^2$ is a hyperparameter of the smoothed classifier $g$ which controls a robustness/accuracy tradeoff. Equivalently, this means that $g(x)$ returns the class $c$ whose decision region $\{x' \in \mathbb{R}^d: f(x') = c\}$ has the largest measure under the distribution $\mathcal{N}(x, \sigma^2 I)$.
\citet{cohen2019certified} recently presented a tight robustness guarantee for the smoothed classifier $g$ and gave Monte Carlo algorithms for certifying the robustness of $g$ around $x$ or predicting the class of $x$ using $g$, that succeed with high probability.

\paragraph{Robustness guarantee for smoothed classifiers}
The robustness guarantee presented by \cite{cohen2019certified} uses the Neyman-Pearson lemma, and is as follows:
suppose that when the base classifier $f$ classifies  $\mathcal{N}(x, \sigma^2 I)$, the class $c_A$ is returned with probability  $p_A =  \mathbb{P}(f(x+\delta) = c_A)$, and the ``runner-up'' class $c_B$ is returned with probability $p_B = \max_{c \neq c_A} \mathbb{P}(f(x+\delta) = c)$.
The smoothed classifier $g$ is robust around $x$ within the radius 
\begin{equation}\label{eq:main_bound}
R = \frac{\sigma}{2} \left(\Phi^{-1}(p_A) - \Phi^{-1}(p_B)\right),
\end{equation}
where $\Phi^{-1}$ is the inverse of the standard Gaussian CDF.
It is not clear how to compute $p_A$ and $p_B$ exactly (if $f$ is given by a deep neural network for example).
Monte Carlo sampling is used to estimate some $\underline{p_A}$ and $\overline{p_B}$ for which $\underline{p_A} \le p_A$ and $\overline{p_B} \ge p_B$ with arbitrarily high probability over the samples.
The result of \eqref{eq:main_bound} still holds if we replace $p_A$ with $\underline{p_A}$ and $p_B$ with $\overline{p_B}$. 

This guarantee can in fact be obtained alternatively by explicitly computing the Lipschitz constant of the smoothed classifier, as we do in Appendix~\ref{appendix:alternative-proof}.

\subsection{\smoothAdv: Attacking smoothed classifiers}\label{sec:smooth-adv-primary-derivation}
\vskip -0.2cm

We now describe our attack against smoothed classifiers.
To do so, it will first be useful to describe smoothed classifiers in a more general setting.
Specifically, we consider a generalization of~\eqref{eq:smoothed-hard} to \emph{soft} classifiers, namely, functions $F: \mathbb{R}^d \to P(\mathcal{Y})$, where $P(\mathcal{Y})$ is the set of probability distributions over $\mathcal{Y}$.
Neural networks typically learn such soft classifiers, then use the argmax of the soft classifier as the final hard classifier.
Given a soft classifier $F$, its associated \textit{smoothed} soft classifier $G: \R^n \to P(\mathcal{Y})$ is defined as 
\begin{equation}
G (x) = \left( F * \normal(0, \sigma^2 I) \right) (x) = \E_{\delta \sim \normal(0, \sigma^2 I)} [F(x + \delta)] \; .
\end{equation}

Let $f(x)$ and $F (x)$ denote the hard and soft classifiers learned by the neural network, respectively, and let $g$ and $G$ denote the associated smoothed hard and smoothed soft classifiers.
Directly finding adversarial examples for the smoothed \textit{hard} classifier $g$ is a somewhat ill-behaved problem because of the argmax, so we instead propose to \textit{find adversarial examples for the smoothed soft classifier} $G$.
Empirically we found that doing so will also find good adversarial examples for the smoothed hard classifier.
More concretely, given a labeled data point $(x, y)$, we wish to find a point $\hat x$ which maximizes the loss of $G$ in an $\ell_2$ ball around $x$ for some choice of loss function.
As is canonical in the literature, we focus on the cross entropy loss $\ell_{\CE}$.
Thus, given a labeled data point $(x, y)$ our (ideal) adversarial perturbation is given by the formula:
\begin{align}
    \hat x &= \argmax_{\|x' - x\|_2 \leq \eps} \ell_\CE (G (x'), y) \nonumber \\ 
    &= \argmax_{\|x' - x\|_2 \leq \eps} \left( - \log \E_{\delta \sim \normal (0, \sigma^2 I)} \left[ \left( F (x' + \delta) \right)_y \right] \right) \; . \label{eq:smooth-adv} \tag{$\mathcal{S}$}
\end{align}
We will refer to~\eqref{eq:smooth-adv} as the \smoothAdv{} objective.
The \smoothAdv{} objective is highly non-convex, so as is common in the literature, we will optimize it via projected gradient descent (PGD), and variants thereof.
It is hard to find exact gradients for~\eqref{eq:smooth-adv}, so in practice we must use some estimator based on random Gaussian samples.
There are a number of different natural estimators for the derivative of the objective function in~\eqref{eq:smooth-adv}, and the choice of estimator can dramatically change the performance of the attack.
For more details, see Section~\ref{sec:first-order}.

We note that \eqref{eq:smooth-adv} should not be confused with the similar-looking objective
\begin{align}
    \hat x_{\mathrm{wrong}} 
&= \argmax_{\|x' - x\|_2 \leq \eps} \left( \E_{\delta \sim \normal (0, \sigma^2 I)} \left[ \ell_\CE (F (x' + \delta), y) \right] \right) \nonumber \\
&= \argmax_{\|x' - x\|_2 \leq \eps} \left( \E_{\delta \sim \mathcal{N}(0, \sigma^2 I)} \left[-\log \left( F(x' + \delta) \right)_y\right] \right) \; , \label{eq:wrong-attack}
\end{align}
as suggested in section G.3 of \cite{cohen2019certified}. 
There is a subtle, but very important, distinction between \eqref{eq:smooth-adv} and \eqref{eq:wrong-attack}. 
Conceptually, solving \eqref{eq:wrong-attack} corresponds to finding an adversarial example of $F$ that is robust to Gaussian noise. 
In contrast, \eqref{eq:smooth-adv} is directly attacking the smoothed model i.e. trying  to find adversarial examples that decrease the probability of correct classification of the smoothed soft classifier $G$. 
From this point of view,~\eqref{eq:smooth-adv} is the right optimization problem that should be used to find adversarial examples of $G$. 
This distinction turns out to be crucial in practice: empirically,  \citet{cohen2019certified} found attacks based on~\eqref{eq:wrong-attack} not to be effective.

Interestingly, for a large class of classifiers, including neural networks, one can alternatively derive the objective~\eqref{eq:smooth-adv} from an optimization perspective, by attempting to directly find adversarial examples to the smoothed hard classifier that the neural network provides.
While they ultimately yield the same objective, this perspective may also be enlightening, and so we include it in Appendix~\ref{sec:alt-deriv}.

\subsection{Adversarial training using \smoothAdv} 
\vskip -0.2cm
We now wish to use our new attack to boost the adversarial robustness of smoothed classifiers.
We do so using the well-studied \emph{adversarial training} framework~\citep{kurakin2016adversarial,madry2017towards}.
In adversarial training, given a current set of model weights $w_t$ and a labeled data point $(x_t, y_t)$, one finds an adversarial perturbation $\hat x_t$ of $x_t$ for the current model $w_t$, and then takes a gradient step for the model parameters, evaluated at the point $(\hat x_t, y_t)$.
Intuitively, this encourages the network to learn to minimize the worst-case loss over a neighborhood around the input.

At a high level, we propose to instead do adversarial training using an adversarial example \emph{for the smoothed classifier}.
We combine this with the approach suggested in~\citet{cohen2019certified}, and train at Gaussian perturbations of this adversarial example.
That is, given current set of weights $w_t$ and a labeled data point $(x_t, y_t)$, we find $\hat x_t$ as a solution to~\eqref{eq:smooth-adv}, and then take a gradient step for $w_t$ based at gaussian perturbations of $\hat x_t$.
In contrast to standard adversarial training, we are training the base classifier so that its associated smoothed classifier minimizes worst-case loss in a neighborhood around the current point.
For more details of our implementation, see Section~\ref{sec:adv-training}.
\textit{We emphasize that although we are training using adversarial examples for the smoothed soft classifier, in the end we certify the robustness of the smoothed hard classifier we obtain after training}.

We make two important observations about our method.
First, adversarial training is an empirical defense, and typically offers no provable guarantees.
However, we demonstrate that by combining our formulation of adversarial training with randomized smoothing, we are able to substantially boost the certifiable robust accuracy of our smoothed classifiers.
Thus, while adversarial training using \smoothAdv{} is still ultimately a heuristic, and offers no provable robustness by itself, the smoothed classifier that we obtain using this heuristic has strong certifiable guarantees.

Second, we found empirically that to obtain strong certifiable numbers using randomized smoothing, it is \emph{insufficient} to use standard adversarial training on the \textit{base} classifier.
While such adversarial training does indeed offer good empirical robust accuracy, the resulting classifier is not optimized for randomized smoothing.
In contrast, our method specifically finds base classifiers whose smoothed counterparts are robust.
As a result, the certifiable numbers for standard adversarial training are noticeably worse than those obtained using our method.
See Appendix~\ref{appendix:attack-base-model} for an in-depth comparison.

\section{Implementing \smoothAdv{} via first order methods}
\label{sec:first-order}
As mentioned above, it is difficult to optimize the \smoothAdv{} objective, so we will approximate it via first order methods.
We focus on two such methods: the well-studied \emph{projected gradient descent (PGD)} method \citep{kurakin2016adversarial, madry2017towards}, and the recently proposed \emph{decoupled direction and norm (DDN)} method~\citep{rony2018decoupling} which achieves  $\ell_2$ robust accuracy competitive with PGD on CIFAR-10.

The main task when implementing these methods is to, given a data point $(x, y)$, compute the gradient of the objective function in~\eqref{eq:smooth-adv} with respect to $x'$.
If we let $J(x') = \ell_{CE} (G (x'), y)$ denote the objective function in~\eqref{eq:smooth-adv}, we have
\begin{equation}
    \label{eq:exact-grad}
    \nabla_{x'} J(x') = \nabla_{x'} \left( - \log \E_{\delta \sim \normal(0, \sigma^2 I)} [F (x' + \delta)_y] \right) \; .
\end{equation}
However, it is not clear how to evaluate~\eqref{eq:exact-grad} exactly, as it takes the form of a complicated high dimensional integral.
Therefore, we will use Monte Carlo approximations.
We sample i.i.d. Gaussians $\delta_1, \ldots, \delta_m \sim \normal (0, \sigma^2 I)$, and use the plug-in estimator for the expectation:
\begin{equation}
    \label{eq:plug-in}
    \nabla_{x'} J(x') \approx \nabla_{x'} \left( - \log \left( \frac{1}{m} \sum_{i = 1}^m  F (x' + \delta_i)_y \right) \right) \;.
\end{equation}
It is not hard to see that if $F$ is smooth, this estimator will converge to~\eqref{eq:exact-grad} as we take more samples.
In practice, if we take $m$ samples, then to evaluate~\eqref{eq:plug-in} on all $m$ samples requires evaluating the network $m$ times.
This becomes  expensive for large $m$, especially if we want to plug this into the adversarial training framework, which is already slow.
Thus, when we use this for adversarial training, we use $m_{\mathrm{train}} \in \{1, 2, 4, 8\}$.
When we run this attack to evaluate the \emph{empirical} adversarial accuracy of our models, we use substantially larger choices of $m$, specifically, $m_{\mathrm{test}} \in \{1, 4, 8, 16, 64, 128\}$.
Empirically we found that increasing $m$ beyond $128$ did not substantially improve performance.

While this estimator does converge to the true gradient given enough samples, note that it is not an unbiased estimator for the gradient.
Despite this, we found that using~\eqref{eq:plug-in} performs very well in practice.
Indeed, using~\eqref{eq:plug-in} yields our strongest empirical attacks, as well as our strongest certifiable defenses when we use this attack in adversarial training.
In the remainder of the paper, we let \smoothAdvPGD{} denote the PGD attack with gradient steps given by~\eqref{eq:plug-in}, and similarly we let \smoothAdvDDN{} denote the DDN attack with gradient steps given by~\eqref{eq:plug-in}.

\subsection{An unbiased, gradient free method}\label{sec:gradient-free} 
\vskip -0.2cm
We note that there is an alternative way to optimize~\eqref{eq:smooth-adv} using first order methods.
Notice that the logarithm in~\eqref{eq:smooth-adv} does not change the argmax, and so it suffices to find a minimizer of $G(x')_y$ subject to the $\ell_2$ constraint.
We then observe that
\begin{equation}\label{eq:grad-free}
    \nabla_{x'} (G(x')_y) = \E_{\delta \sim \normal(0, \sigma^2 I)} \left[ \nabla_{x'} F(x' + \delta)_y \right] \stackrel{(a)}{=} \E_{\delta \sim \normal (0, \sigma^2 I)} \left[ \frac{\delta}{\sigma^2} \cdot F(x' + \delta)_y \right] \; .
\end{equation}
The equality (a) is known as \emph{Stein's lemma}~\cite{stein1981estimation}, although we note that something similar can be derived for more general distributions.
There is a natural unbiased estimator for~\eqref{eq:grad-free}: sample i.i.d. gaussians $\delta_1, \ldots, \delta_m \sim \normal (0, \sigma^2 I)$, and form the estimator
$    \nabla_{x'} (G(x')_y) \approx \frac{1}{m} \sum_{i = 1}^m \frac{\delta_i}{\sigma^2} \cdot F (x' + \delta_i)_y \; .$
This estimator has a number of nice properties.
As mentioned previously, it is an unbiased estimator for~\eqref{eq:grad-free}, in contrast to~\eqref{eq:plug-in}.
It also requires no computations of the gradient of $F$; if $F$ is a neural network, this saves both time and memory by not storing preactivations during the forward pass.
Finally, it is very general: the derivation of~\eqref{eq:grad-free} actually holds even if $F$ is a hard classifier (or more precisely, the one-hot embedding of a hard classifier).
In particular, this implies that this technique can even be used to directly find adversarial examples of the smoothed hard classifier.

Despite these appealing features, in practice we find that this attack is quite weak. 
We speculate that this is because the variance of the gradient estimator is too high.
For this reason, in the empirical evaluation we focus on attacks using~\eqref{eq:plug-in}, but we believe that investigating this attack in practice is an interesting direction for future work. See Appendix~\ref{appendix:grad-free} for more details.

\begin{algorithm}[t]
   \caption{1: \textsc{SmoothAdv}-ersarial Training}
   \label{pseudocode-smoothadv}
   \begin{algorithmic}
   \STATE \textbf{function} \textsc{TrainMiniBatch}($(x^{(1)}, y^{(1)})$, $(x^{(2)}, y^{(2)})$, \ldots, $(x^{(B)}, y^{(B)})$)
   \STATE \quad \textsc{Attacker} $\gets$ (\smoothAdvPGD{} or \smoothAdvDDN{})
   \STATE \quad Generate noise samples $\delta_i^{(j)} \sim \mathcal{N}(0, \sigma^2 I)$ for $1 \leq i \leq m$,
   $1 \leq j \leq B$
   \STATE \quad $L \gets []$ {\it \quad \# List of adversarial examples for training}
   \STATE \quad \textbf{for} $1 \leq j \leq B$ \textbf{do}
   \STATE \quad \quad $\hat x^{(j)} \gets x^{(j)}$ {\it \quad \# Adversarial example}
   \STATE \quad \quad \textbf{for} $1 \leq k \leq T$ \textbf{do}
   \STATE \quad \quad \quad Update $\hat x^{(j)}$ according
   to the $k$-th step of \textsc{Attacker},
   where we use 
   \STATE \quad\quad\quad the noise samples $\delta_1^{(j)}$,
   $\delta_2^{(j)}$, \ldots, $\delta_m^{(j)}$ to estimate
   a gradient of the loss of the smoothed
   \STATE \quad\quad\quad model
   according to~(\ref{eq:plug-in})
   \STATE \quad \quad \quad {\it \# We are reusing the same
   noise samples between different steps of the attack}
   \STATE \quad \quad \textbf{end}
   \STATE \quad \quad Append $((\hat x^{(j)} + \delta_1^{(j)}, y^{(j)}), (\hat x^{(j)} + \delta_2^{(j)}, y^{(j)}), \ldots, (\hat x^{(j)} + \delta_m^{(j)}, y^{(j)}))$ to $L$
   \STATE \quad \quad {\it \# Again, we are reusing the same noise samples
   for the augmentation}
   \STATE \quad \textbf{end}
   \STATE \quad Run backpropagation on $L$
   with an appropriate learning rate
\end{algorithmic}
\end{algorithm}

\subsection{Implementing adversarial training for smoothed classifiers}
\label{sec:adv-training}
\vskip -0.2cm
We incorporate adversarial training into the approach of~\citet{cohen2019certified} changing as few moving parts as possible
in order to enable a direct comparison. In particular,
we use the same network architectures, batch size, and learning rate schedule.
For CIFAR-10, we change the number of epochs, but for ImageNet, we leave it the same. We discuss more of these specifics in Appendix~\ref{appendix:experiments-details}, and here we describe how to perform adversarial training on a single mini-batch.
The algorithm is shown in Pseudocode~\ref{pseudocode-smoothadv}, with the following parameters:
$B$ is the mini-batch size, $m$ is the number of noise samples
used for gradient estimation in~(\ref{eq:plug-in}) as 
well as for Gaussian noise data augmentation, and $T$ is the number of steps of an attack\footnote{Note that we are reusing the same noise samples during every step of our attack as well as during augmentation. Intuitively, this helps to stabilize the attack process.}.

 \enlargethispage{\baselineskip}
 
\section{Experiments}
\vskip -0.3cm
We primarily compare with \citet{cohen2019certified} as it was shown to outperform all other scalable provable $\ell_2$ defenses by a wide margin.
As our experiments will demonstrate, our method consistently and significantly outperforms \citet{cohen2019certified} even further, establishing the state-of-the-art for provable $\ell_2$-defenses.
We run experiments on ImageNet \citep{deng2009imagenet} and  CIFAR-10 \citep{krizhevsky2009learning}. We use the same base classifiers $f$ as \citet{cohen2019certified}: a ResNet-50 \citep{he2016deep} on ImageNet, and ResNet-110 on CIFAR-10.
Other than the choice of attack (\smoothAdvPGD\ or \smoothAdvDDN) for adversarial training, our experiments are distinguished based on five main hyperparameters:
\begin{equation}
\small
\begin{aligned}
        \epsilon &= \text{maximum allowed $\ell_2$ perturbation of the input}\\[-4pt]
    T &= \text{number of steps of the attack}\\[-4pt]    
    \sigma &= \text{std. of Gaussian noise data augmentation during training and certification}\\[-4pt]
    m_{\mathrm{train}} &= \text{number of noise samples used to estimate \eqref{eq:plug-in} during training}\\[-4pt]
    m_{\mathrm{test}} &= \text{number of noise samples used to estimate \eqref{eq:plug-in} during evaluation}
\end{aligned}
\tag{$\Diamond$}
\label{eqn:hyper}
\end{equation}

Given a smoothed classifier $g$, we use the same prediction and certification algorithms, \textsc{Predict} and \textsc{Certify}, as \cite{cohen2019certified}.
Both algorithms sample base classifier predictions under Gaussian noise.
\textsc{Predict} outputs the majority vote if the vote count passes a binomial hypothesis test, and abstains otherwise.
\textsc{Certify} certifies the majority vote is robust if the fraction of such votes is higher by a calculated margin than the fraction of the next most popular votes, and abstains otherwise.
For details of these algorithms, we refer the reader to \cite{cohen2019certified}.

The \textbf{certified accuracy} at radius $r$ is defined as the fraction of the test set which $g$  classifies correctly (without abstaining) and certifies robust at an $\ell_2$ radius $r$.
Unless otherwise specified, we use the same $\sigma$ for certification as the one used for training the base classifier $f$. Note that $g$ is a randomized smoothing classifier, so this reported accuracy is \textit{approximate}, but can get arbitrarily close to the \textit{true} certified accuracy as the number of samples of $g$ increases (see \cite{cohen2019certified} for more details).
Similarly, the \textbf{empirical accuracy} is defined as the fraction of the $\ell_2$ \smoothAdv-ersarially attacked test set which $g$ classifies correctly (without abstaining). 
Both \textsc{Predict} and \textsc{Certify} have a parameter $\alpha$ defining the failure rate of these algorithms. 
Throughout the paper, we set $\alpha =0.001$ (similar to \cite{cohen2019certified}), which means there is at most a 0.1\% chance that \textsc{Predict} \textit{does not} return the most probable class under the smoothed classifier $g$, or that \textsc{Certify} falsely certifies a non-robust input.

\begin{figure*}[t]
	\includegraphics[width=0.32\textwidth]{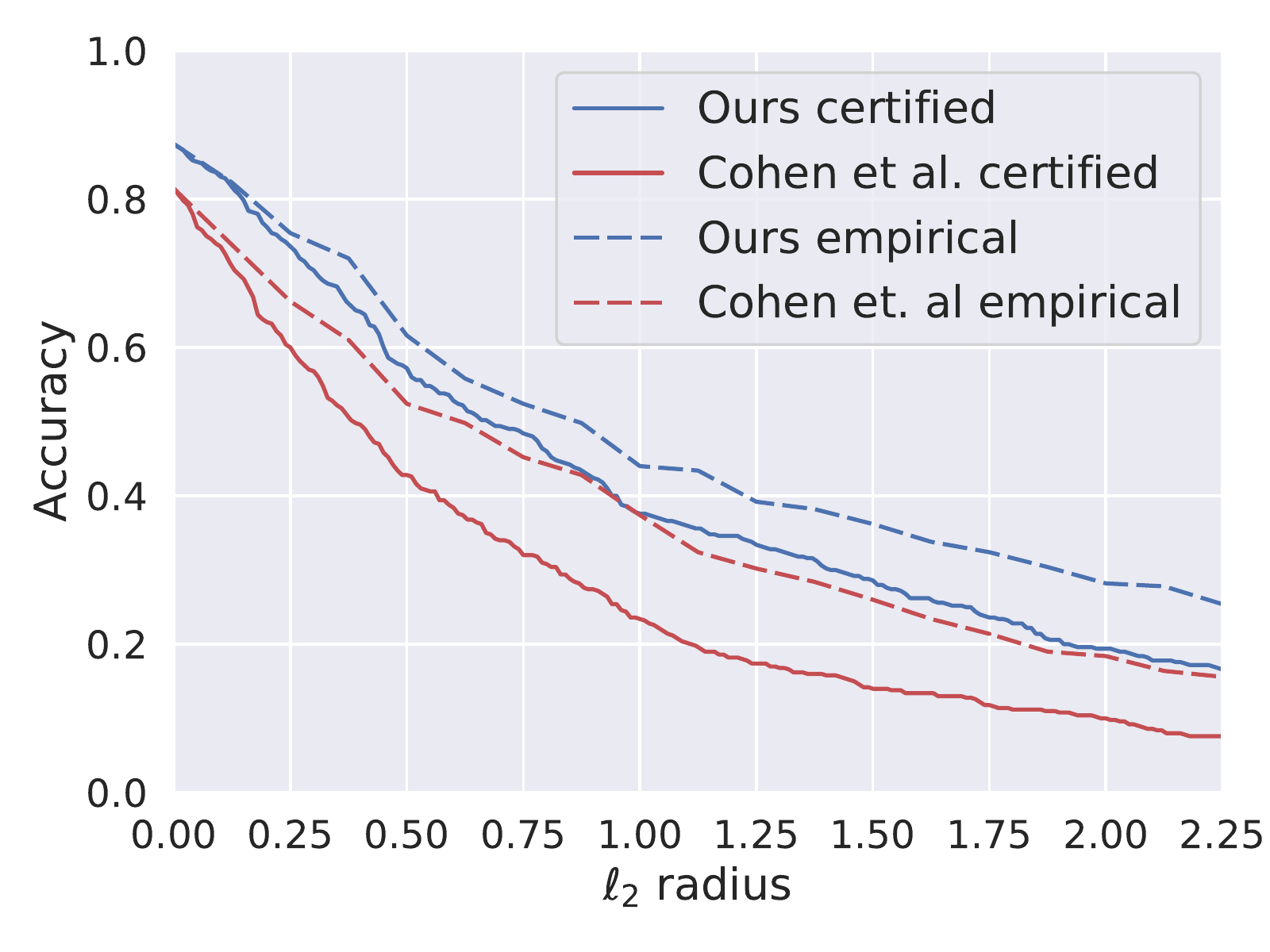}
	\includegraphics[width=0.32\textwidth]{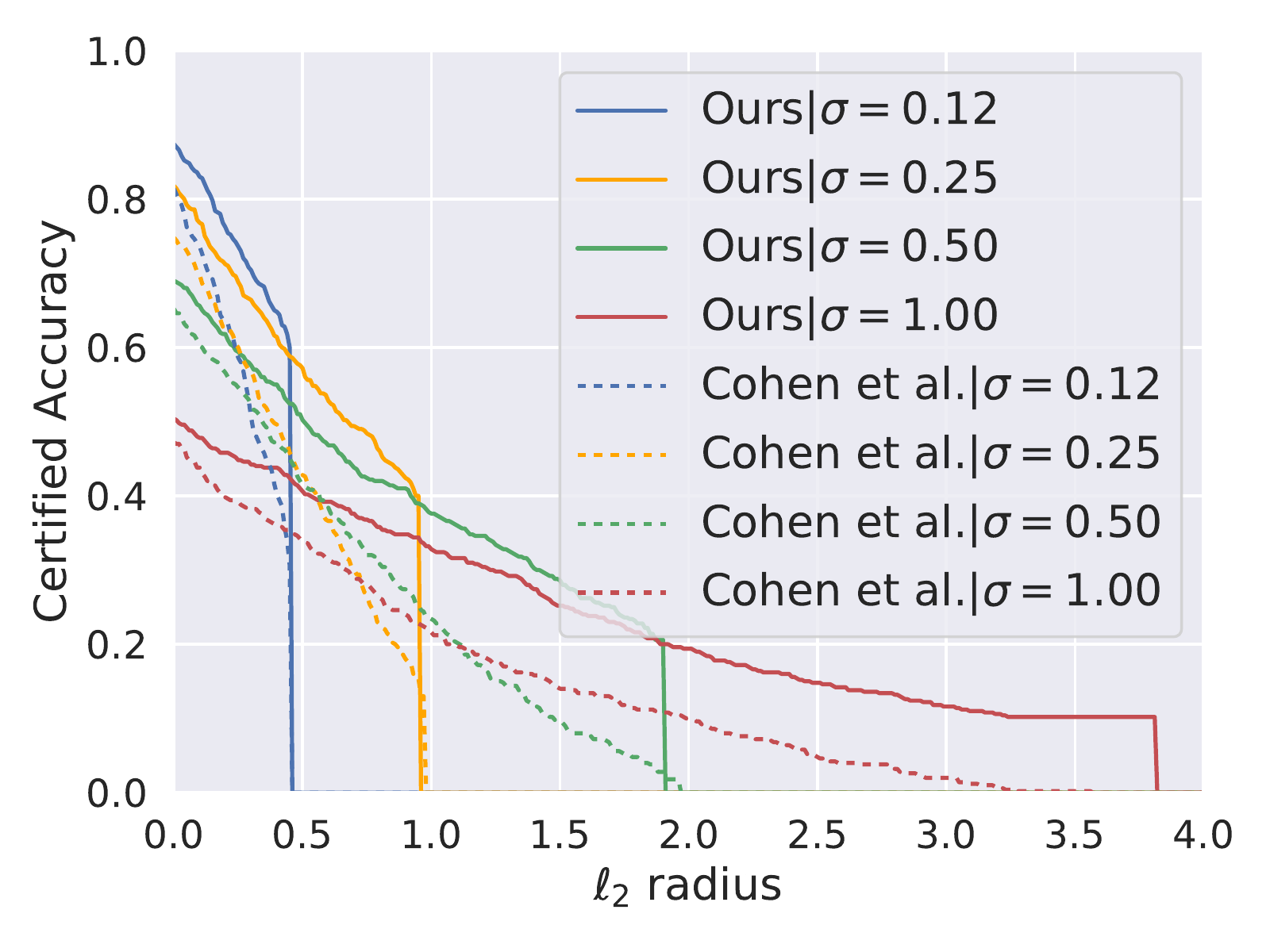}
	\includegraphics[width=0.32\textwidth]{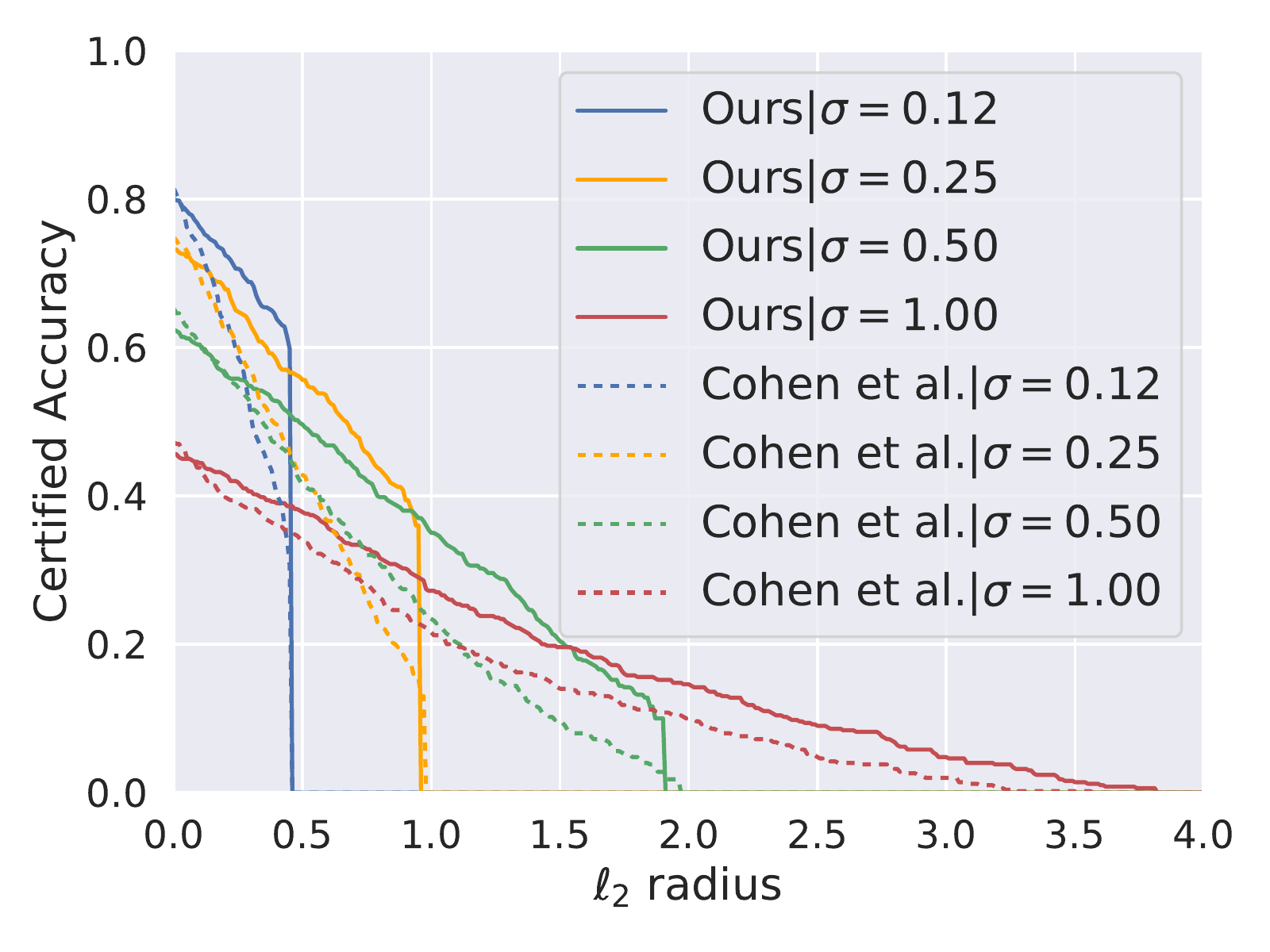}
\caption{Comparing our \smoothAdv-ersarially trained CIFAR-10 classifiers vs \citet{cohen2019certified}.
\textbf{(Left)} Upper envelopes of certified accuracies over all experiments.
\textbf{(Middle)} Upper envelopes of certified accuracies per $\sigma$.
\textbf{(Right)} Certified accuracies of one representative model per $\sigma$.
Details of each model used to generate these plots and their certified accuracies are in Tables~\ref{table:cifar_certify_results_PGD_DDN_2steps_1samples}-\ref{table:cifar_certify_results_PGD_10steps_2/4/8samples} in Appendix~\ref{appendix:detialed_certification_results}.
\label{fig:cifar}}
\end{figure*}

\begin{figure*}[t]
\includegraphics[width=0.32\textwidth]{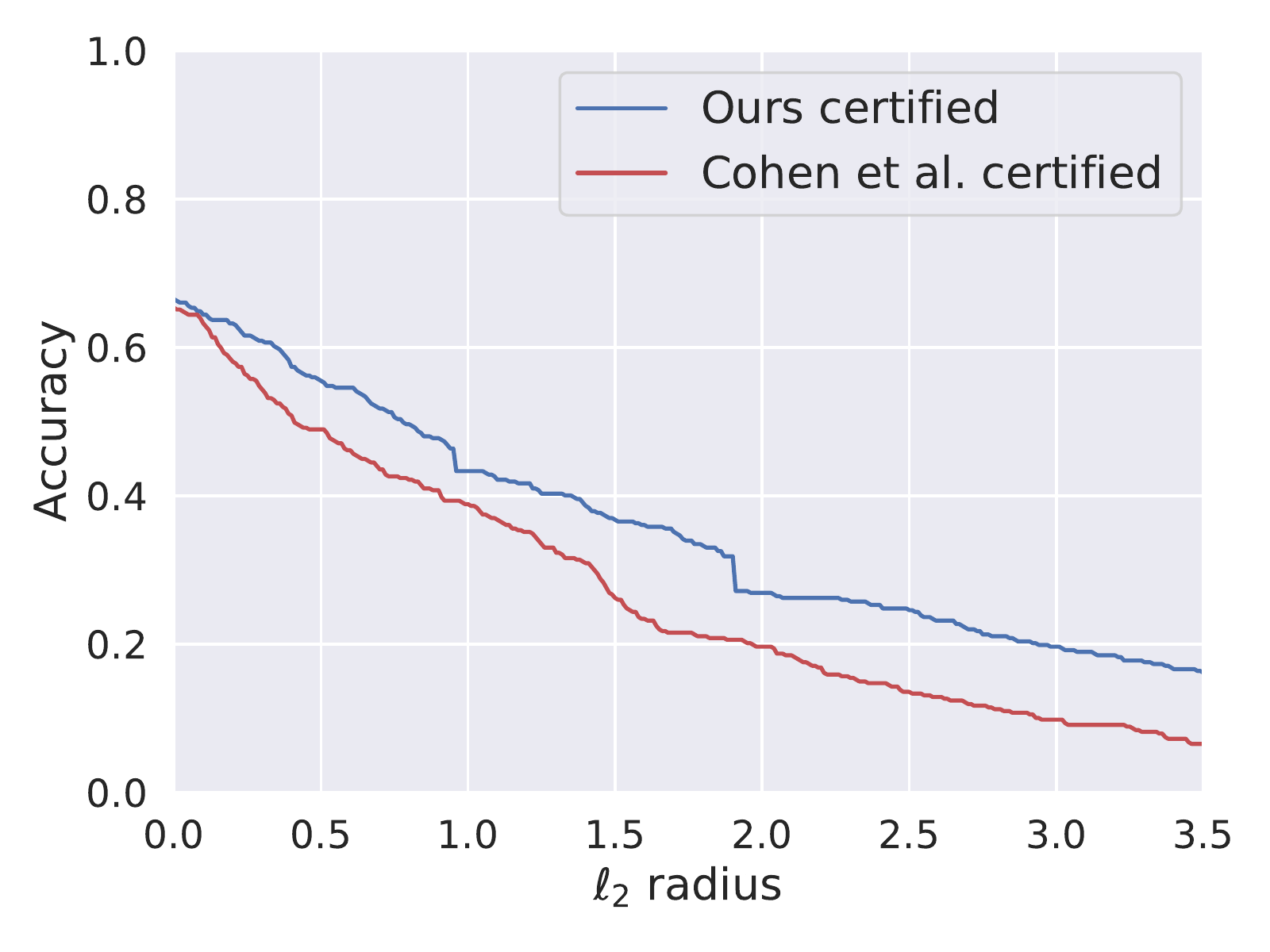}
\includegraphics[width=0.32\textwidth]{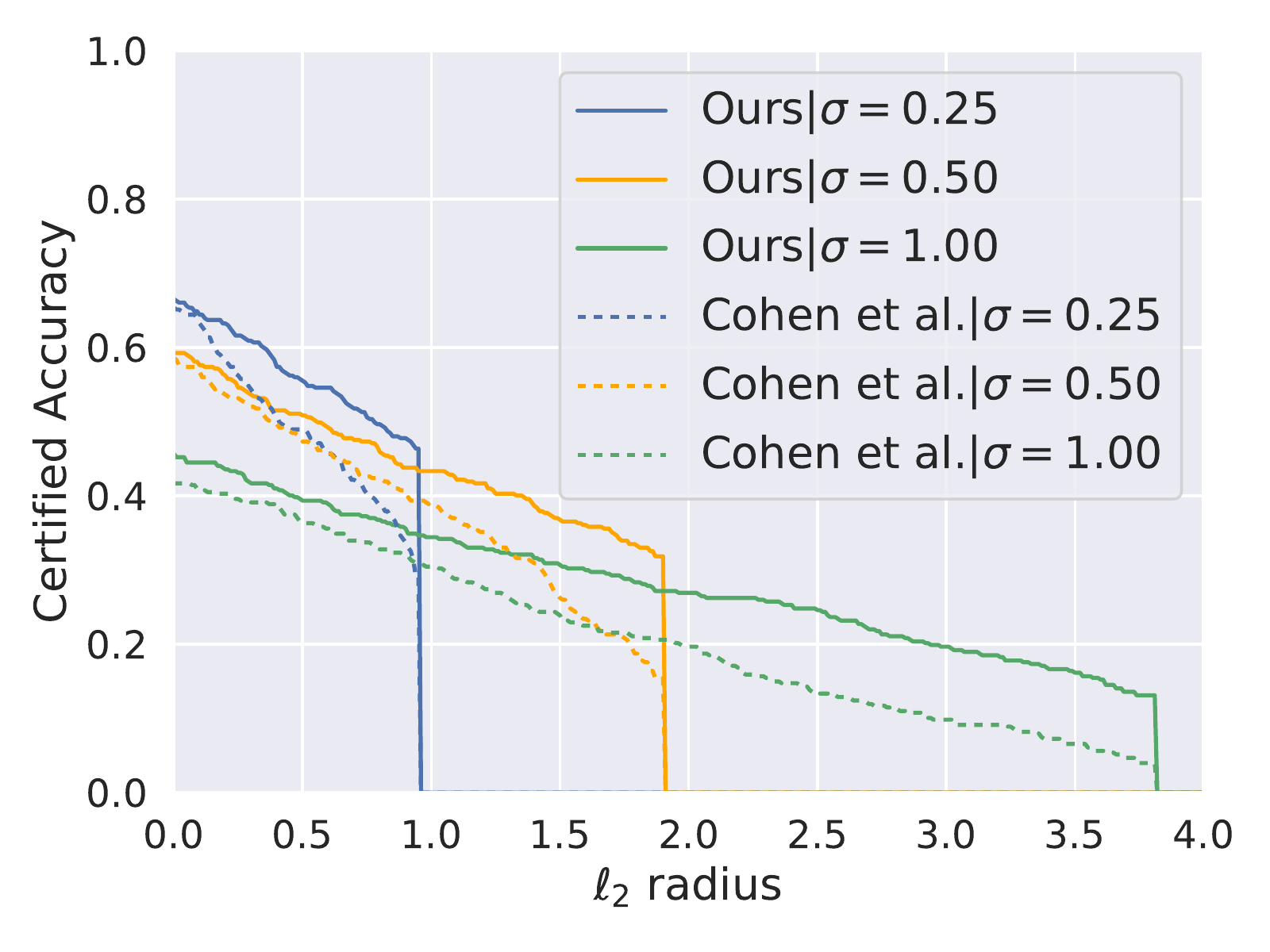}
\includegraphics[width=0.32\textwidth]{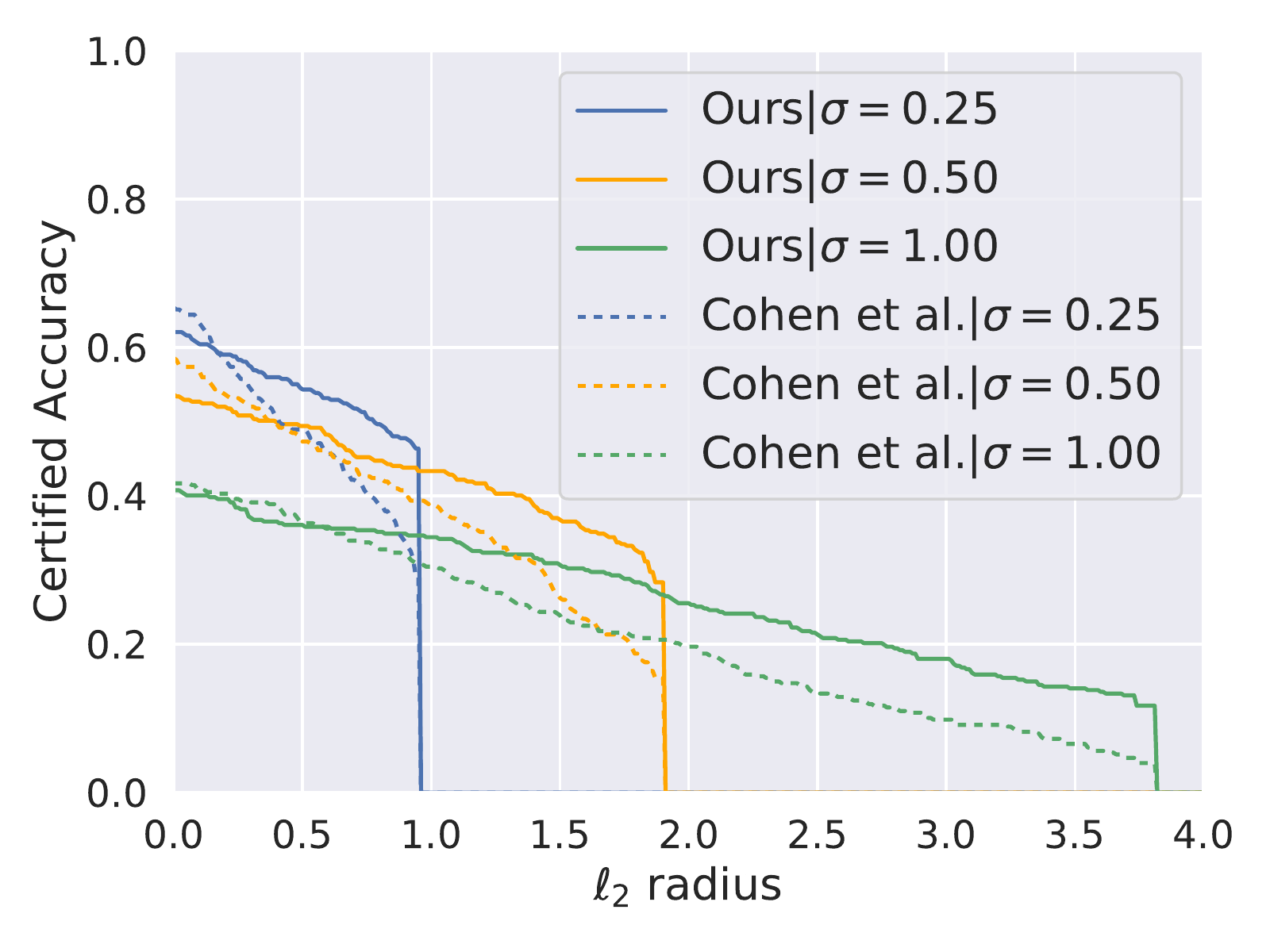}
\caption{Comparing our \smoothAdv-ersarially trained ImageNet classifiers vs \citet{cohen2019certified}.
Subfigure captions are same as Fig.~\ref{fig:cifar}.
Details of each model used to generate these plots and their certified accuracies are in Table~\ref{table:imagenet_certify_results} in Appendix~\ref{appendix:detialed_certification_results}.}
\label{fig:imagenet_results}
\end{figure*}

\subsection{\smoothAdv-ersarial training}\label{sec:adv-training-results}
\vskip -0.2cm
To assess the effectiveness of our method, we learn a smoothed classifier $g$ that is adversarial trained using  \eqref{eq:smooth-adv}. Then we compute the  \textit{certified accuracies}\footnote{Similar to \citet{cohen2019certified}, we certified the full CIFAR-10 test set and a subsampled  ImageNet test set of 500 samples.} over a range of $\ell_2$ radii $r$.
Tables~\ref{imagenet-certified-accuracy} and \ref{cifar-certified-accuracy} report the certified accuracies using our method compared to \citep{cohen2019certified}.
For all radii, we outperform the certified accuracies of \citep{cohen2019certified} by a significant margin on both ImageNet and CIFAR-10. These results are elaborated below.

\paragraph{For CIFAR-10} Fig.~\ref{fig:cifar}(left) plots the upper envelope of the certified accuracies that we get by choosing the best model for each radius over a grid of hyperparameters.
This grid consists of $m_{train} \in \{1, 2, 4, 8\}$, $\sigma \in \{0.12, 0.25, 0.5,1.0\}$, $\eps \in \{0.25, 0.5, 1.0, 2.0\}$ (see \ref{eqn:hyper} for explanation), and one of the following attacks \{\smoothAdvPGD{}, \smoothAdvDDN{}\} with $T \in \{2, 4 , 6, 8, 10\}$ steps. The certified accuracies of each model can be found in Tables~\ref{table:cifar_certify_results_PGD_DDN_2steps_1samples}-\ref{table:cifar_certify_results_PGD_10steps_2/4/8samples} in Appendix~\ref{appendix:detialed_certification_results}. These results are compared to those of \citet{cohen2019certified} by plotting their reported certified accuracies. Fig.~\ref{fig:cifar}(left) also plots the corresponding empirical accuracies using \smoothAdvPGD{} with $m_{test} = 128$. Note that \textit{our \textbf{certified} accuracies are higher than the \textbf{empirical} accuracies of \citet{cohen2019certified}}. 

Fig.~\ref{fig:cifar}(middle) plots our vs \citep{cohen2019certified}'s best models for varying noise level $\sigma$. Fig.~\ref{fig:cifar}(right) plots a representative model for each $\sigma$ from our adversarially trained models. Observe that we outperform \citep{cohen2019certified} in all three plots.

\paragraph{For ImageNet} The results are summarized in Fig.~\ref{fig:imagenet_results}, which is similar to  Fig.~\ref{fig:cifar} for CIFAR-10, with the difference being the set of smoothed models we certify. This set includes smoothed models trained using $m_{\mathrm{train}}=1$, $\sigma \in \{0.25, 0.5,1.0\}$, $\eps \in \{0.5, 1.0, 2.0, 4.0\}$, and one of the following attacks \{1-step \smoothAdvPGD{}, 2-step \smoothAdvDDN{}\}. Again, our models outperform those of \citet{cohen2019certified} overall and per $\sigma$ as well. The certified accuracies of each model can be found in Table~\ref{table:imagenet_certify_results} in Appendix~\ref{appendix:detialed_certification_results}.

We point out, as mentioned by \citet{cohen2019certified}, that $\sigma$  controls a robustness/accuracy trade-off.
When $\sigma$ is low, small radii can be certified with high accuracy, but large radii cannot be certified at all. When $\sigma$ is high, larger radii can be certified, but smaller radii are certified at a lower accuracy. This can be observed in the middle and the right plots of Fig.~\ref{fig:cifar} and \ref{fig:imagenet_results}.

\paragraph{Effect on clean accuracy}
Training smoothed classifers using \smoothAdv{} as shown improves upon the certified accuracy of \citet{cohen2019certified} for each $\sigma$, although this comes with the well-known effect of adversarial training in decreasing the standard accuracy, so we sometimes see small drops in the accuracy at $r=0$, as observed in Fig.~\ref{fig:cifar}(right) and \ref{fig:imagenet_results}(right).

\begin{table}[t]
\caption{Certified $\ell_{\infty}$ robustness at a radius of $\frac{2}{255}$ on CIFAR-10. Note that our models and \citet{carmon2019unlabeled}'s give accuracies with high probability (\textsc{w.h.p}).}
\label{table:certified_linf}
\begin{center}
\begin{small}
\begin{sc}
\begin{tabular}{l|cc}
Model                 & $\ell_{\infty}$ Acc. at $2/255$ & Standard Acc. \\
\midrule
Ours      (\%)                           &           \textbf{68.2} (w.h.p)                 &   \textbf{86.2} (w.h.p)   \\
\citet{carmon2019unlabeled} (\%)         &           $63.8\pm 0.5$ (w.h.p)  &   $80.7\pm0.3$ (w.h.p)     \\
\citet{wong2018provable} (single) (\%)  &           53.9               &    68.3    \\
\citet{wong2018provable} (ensemble) (\%) &           63.6               &    64.1    \\
IBP \cite{gowal2018effectiveness} (\%) &            50.0              &     70.2
\end{tabular}
\end{sc}
\end{small}
\end{center}
\end{table}

\paragraph{$\ell_{2}$ to $\ell_{\infty}$ certified defense}
Since the $\ell_2$ ball of radius $\sqrt d$ contains the $\ell_\infty$ unit ball in $\R^d$, a model robust against $\ell_2$ perturbation of radius $r$ is also robust against $\ell_\infty$ perturbation of norm $r/\sqrt d$.
Via this naive conversion, we find our $\ell_2$-robust models enjoy non-trivial $\ell_{\infty}$ certified robustness. In~Table~\ref{table:certified_linf}, we report the best\footnote{We report the model with the highest certified $\ell_2$ accuracy on CIFAR-10 at a radius of 0.435, amongst all our models trained in this paper.} $\ell_{\infty}$ certified accuracy that we get on CIFAR-10 at a radius of 2/255 (implied by the $\ell_{2}$ certified accuracy at a radius of $0.435 \approx 2\sqrt{3\times 32^2} / 255$).
We exceed previous state-of-the-art in certified $\ell_{\infty}$ defenses by at least $3.9\%$.
We obtain similar results for ImageNet certified $\ell_{\infty}$ defenses at a radius of $1/255$ where we exceed the previous state-of-the-art by $8.2\%$; details are in appendix~\ref{appendix:l2_to_linf}.

\paragraph{Additional experiments and observations}
We compare the effectiveness of smoothed classifiers when they are trained \smoothAdv-versarially vs. when their \textit{base} classifier is trained via standard adversarial training (we will refer to the latter as \emph{vanilla adversarial training}).
As expected, because the training objective of \smoothAdv-models aligns with the actual certification objective,
those models achieve noticeably more certified robustness over all radii compared to smoothed classifiers resulting from vanilla adversarial training.
We defer the results and details to Appendix~\ref{appendix:attack-base-model}.

Furthermore, \smoothAdv{} requires the evaluation of \eqref{eq:plug-in} as discussed in Section~\ref{sec:first-order}. We analyze in Appendix~\ref{appendix:effect_of_m_on_training} how the number of Gaussian noise samples $m_{\mathrm{train}}$, used in \eqref{eq:plug-in} to find adversarial examples, affects the robustness of the resulting smoothed models.
As expected, we observe that models trained with higher $m_{\mathrm{train}}$ tend to have higher certified accuracies.

Finally, we analyze the effect of the maximum allowed $\ell_2$ perturbation $\eps$ used in \smoothAdv{} on the robustness of smoothed models in Appendix~\ref{appendix:effect_of_eps_on_training}. We observe that as $\eps$ increases, the certified accuracies for small $\ell_2$ radii decrease, but those for large $\ell_2$ radii increase, which is expected.

\subsection{More Data for Better Provable Robustness}\label{sec:more-data}
We explore using more data to improve the robustness of smoothed classifiers. Specifically, we pursue two ideas: 1) \textit{pre-training} similar to \cite{hendrycks2019using}, and 2) \textit{semi-supervised learning} as in \cite{carmon2019unlabeled}.

\paragraph{Pre-training} \citet{hendrycks2019using} recently showed that using pre-training can improve the adversarial robustness of classifiers, and achieved state-of-the-art results for \textit{empirical} $l_\infty$ defenses on CIFAR-10 and CIFAR-100. We employ this within our framework; we pretrain smoothed classifiers on ImageNet, then fine-tune them on CIFAR-10. Details can be found in Appendix~\ref{appendix:pretraining}.

\paragraph{Semi-supervised learning} \citet{carmon2019unlabeled} recently showed that using unlabelled data can improve the adversarial robustness as well. They employ a simple, yet effective, semi-supervised learning technique called \emph{self-training} to improve the robustness of CIFAR-10 classifiers. We employ this idea in our framework and we train our CIFAR-10 smoothed classifiers via self-training using the unlabelled dataset used in \citet{carmon2019unlabeled}. Details can be found in Appendix~\ref{appendix:semisupervision}.

We further experiment with combining semi-supervised learning and pre-training, and the details are in Appendix~\ref{appendix:semisupervision_and_pretraining}. We observe consistent improvement in the certified robustness of our smoothed models when we employ pre-training or semi-supervision. The results are summarized in Table~\ref{cifar-certified-accuracy}.

\subsection{Attacking trained models with \smoothAdv}\label{sec:adv-attack-results}
\vskip -0.2cm

\begin{figure*}[t]
\begin{center}
\includegraphics[width=0.7\textwidth]{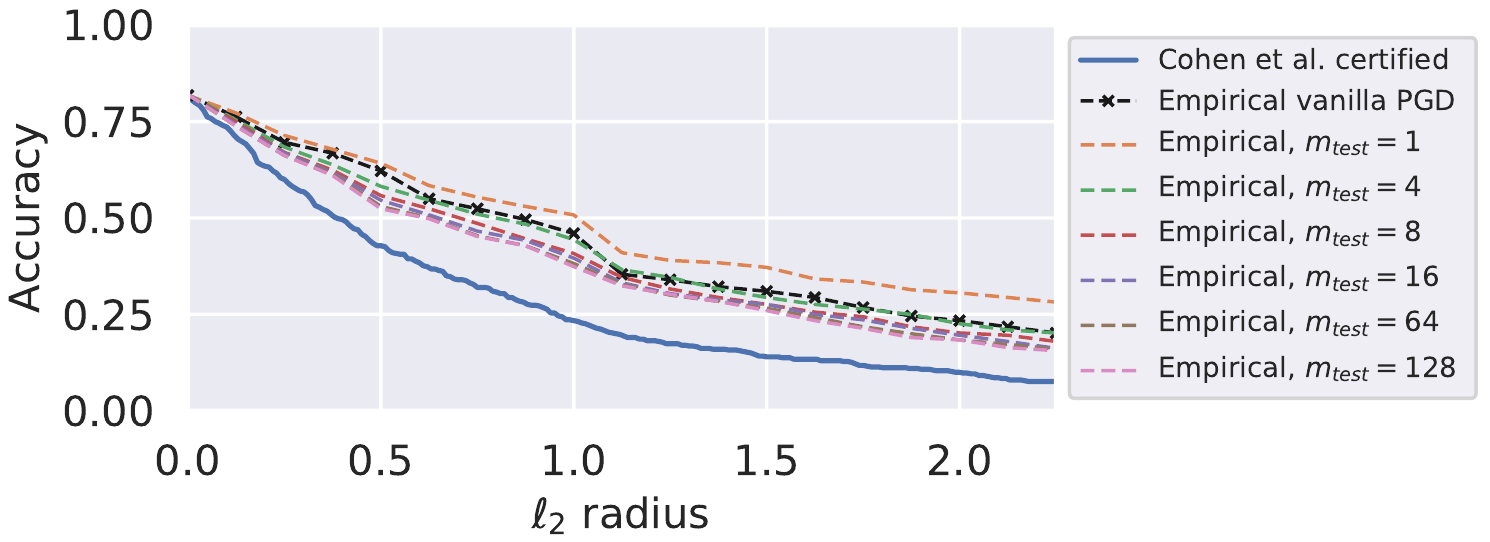}
\caption{Certified and empirical robust accuracy of \citet{cohen2019certified}'s models on CIFAR-10. For each $\ell_2$ radius $r$, the certified/empirical accuracy is the maximum over randomized smoothing models trained using $\sigma \in \{0.12, 0.25, 0.5, 1.0\}$. The empirical accuracies are found using 20 steps of \smoothAdvPGD{}. The closer an empirical curve is to the certified curve, the stronger the corresponding attack is (the lower the better).}
\label{fig:empirical_cohen_vary_m}
\end{center}
\end{figure*}

In this section, we assess the performance of our attack, particularly \smoothAdvPGD{}, for finding adversarial examples for the CIFAR-10 randomized smoothing models of \citet{cohen2019certified}.

\smoothAdvPGD{} requires the evaluation of \eqref{eq:plug-in} as discussed in Section~\ref{sec:first-order}. Here, we analyze how sensitive our attack is to the number of samples $m_{\mathrm{test}}$ used in \eqref{eq:plug-in} for estimating the gradient of the adversarial objective. Fig.~\ref{fig:empirical_cohen_vary_m}  shows the empirical accuracies for various values of $m_{\mathrm{test}}$. Lower accuracies corresponds to stronger attack.
\smoothAdv{} with $m_{\mathrm{test}}=1$ sample performs worse than the vanilla PGD attack on the base classifier, but as $m_{\mathrm{test}}$ increases, our attack becomes stronger, decreasing the gap between certified and empirical accuracies.
We did not observe any noticeable improvement beyond $m_{\mathrm{test}}=128$. More details are in Appendix~\ref{appendix:effect_of_m_on_attack}.

While as discussed here, the success rate of the attack is affected by the number of Gaussian noise samples $m_{\mathrm{test}}$ used by the attacker, it is also affected by the number of Gaussian noise samples $n$ in \textsc{Predict} used by the classifier.
Indeed, as $n$ increases, abstention due to low confidence becomes more rare, increasing the prediction quality of the smoothed classifier.
See a detailed analysis in Appendix~\ref{appendix:effect_of_n_on_predict}.

\vspace{-0.2cm}
\section{Related Work}
\vspace{-0.2cm}
Recently, many approaches (defenses) have been proposed to build adversarially robust classifiers, and these approaches can be broadly divided into \emph{empirical} defenses and \emph{certified} defenses.

\textbf{Empirical defenses} are empirically robust to existing adversarial attacks, and the best empirical defense so far is \emph{adversarial training}  \citep{ kurakin2016adversarial, madry2017towards}. In this kind of defense, a neural network is trained to minimize the worst-case loss over a neighborhood around the input.
Although such defenses seem powerful, nothing guarantees that a more powerful, not yet known, attack would not break them; the most that can be said is that known attacks are unable to find adversarial examples around the data points.
In fact, most empirical defenses proposed in the literature were later ``broken'' by stronger adversaries \citep{carlini2017adversarial, athalye2018obfuscated, uesato2018adversarial, athalye2018robustness}.
To stop this arms race between defenders and attackers, a number of work tried to focus on building certified defenses which enjoy formal robustness guarantees.

\textbf{Certified defenses} are provably robust to a specific class of adversarial perturbation, and can guarantee that for any input $x$, the classifier's prediction is constant within a neighborhood of $x$. These are typically based on certification methods which are either \emph{exact} (a.k.a ``complete'') or \emph{conservative} (a.k.a ``sound but incomplete''). Exact methods, usually based on Satisfiability Modulo Theories solvers \citep{katz2017reluplex, ehlers2017formal} or mixed integer linear programming \citep{tjeng2018evaluating, lomuscio2017approach, fischetti2017deep}, are guaranteed to find an adversarial example around a datapoint if it exists. Unfortunately, they are computationally inefficient and difficult to scale up to large neural networks. Conservative methods are also guaranteed to detect an adversarial example if exists, but they might mistakenly flag a safe data point as vulnerable to adversarial examples. On the bright side, these methods are more scalable and efficient which makes some  of them useful for building certified defenses  \citep{wong2018provable, wang2018mixtrain, wang2018efficient, raghunathan2018certified, raghunathan2018semidefinite, wong2018scaling, dvijotham2018dual, dvijotham2018training, croce2018provable, salman2019convex, gehr2018ai2, mirman2018differentiable, singh2018fast, gowal2018effectiveness, weng2018towards, zhang2018efficient}. However, none of them have yet been shown to scale to practical networks that are large and expressive enough to perform well on ImageNet, for example. To scale up to practical networks, randomized smoothing has been proposed as a \textit{probabilistically} certified defense.

\vspace{-0.2cm}
\paragraph{Randomized smoothing}
A randomized smoothing classifier is not itself a neural network, but uses a neural network as its base for classification.
Randomized smoothing was proposed by several works \citep{liu2018towards, cao2017mitigating} as a heuristic defense without proving any guarantees.
\citet{lecuyer2018certified} first proved robustness guarantees for randomized smoothing classifier, utilizing inequalities from the differential privacy literature.
Subsequently, \citet{li2018second} gave a stronger robustness guarantee using tools from information theory. Recently, \citet{cohen2019certified} provided a tight robustness guarantee for randomized smoothing and consequently achieved the state of the art in $\ell_2$-norm certified defense.

\vskip -0.6cm
\section{Conclusions}
\vspace{-0.3cm}
In this paper, we designed an adapted attack for smoothed classifiers, and we showed how this attack can be used in an adversarial training setting to substantially improve the provable robustness of smoothed classifiers. 
We demonstrated through extensive experimentation that our adversarially trained smooth classifiers consistently outperforms all existing provably $\ell_2$-robust classifiers by a significant margin on ImageNet and CIFAR-10, establishing the state of the art for provable $\ell_2$-defenses.
\newpage

\section*{Acknowledgements}

We would like to thank Zico Kolter, Jeremy Cohen, Elan Rosenfeld, Aleksander Madry, Andrew Ilyas, Dimitris Tsipras, Shibani Santurkar, and Jacob Steinhardt for comments and discussions.

\bibliography{robust_training}
\bibliographystyle{plainnat}

\clearpage
\appendix

\section{Alternative proof of the robustness guarantee of \citet{cohen2019certified} via explicit Lipschitz constants of smoothed classifier}\label{appendix:alternative-proof}
In this appendix, we present an alternate derivation of \eqref{eq:main_bound}.
Fix $f : \R^n \rightarrow [0,1]$ and define $\hat{f}$ by:
\[
\hat{f}(x) = \left( f * \normal(0, I) \right) (x) = \frac1{(2\pi)^{n/2}} \int_{\R^n} f(t) \exp\left(- \frac12 \|x-t\|^2 \right) dt \,.
\]
The smoothed function $\hat f$ is known as the \emph{Weierstrass transform} of $f$, and a classical property of the Weierstrass transform is its induced smoothness, as demonstrated by the following.
\begin{lemma}
The function $\hat{f}$ is $\sqrt{\frac{2}{\pi}}$-Lipschitz.
\end{lemma}

\begin{proof}
It suffices to prove that for any unit direction $u$ one has $u \cdot \nabla \hat{f}(x) \leq \sqrt{\frac{2}{\pi}}$. Note that:
\begin{equation} \label{eq:steinlemmalol}
\nabla \hat{f}(x) = \frac1{(2\pi)^{n/2}} \int_{\R^n} f(t) (x-t) \exp\left(- \frac12 \|x-t\|^2 \right) dt \,,
\end{equation}
and thus (using $|f(t)| \leq 1$, and classical integration of the Gaussian density)
\begin{eqnarray*}
u \cdot \nabla \hat{f}(x) & \leq & \frac1{(2\pi)^{n/2}} \int_{\R^n} |u \cdot (x-t)| \exp\left(- \frac12 \|x-t\|^2 \right) dt \\
& = & \frac{1}{\sqrt{2 \pi}} \int_{-\infty}^{+ \infty} |s| \exp\left(-\frac12 s^2 \right) ds = \sqrt{\frac{2}{\pi}} \,.
\end{eqnarray*}
\end{proof}

However, $\hat{f}$ in fact satisfies an even stronger \emph{nonlinear} smoothness property as shown in the following lemma.

\begin{lemma}
Let $\Phi(a) = \frac{1}{\sqrt{2 \pi}} \int_{-\infty}^a \exp\left( - \frac12 s^2 \right) ds$. For any function $f : \R^n \rightarrow [0,1]$, the map $x \mapsto \Phi^{-1}(\hat{f}(x))$ is $1$-Lipschitz.
\end{lemma}

\begin{proof}
Note that:
\[
\nabla \Phi^{-1}(\hat{f}(x)) = \frac{\nabla \hat{f}(x)}{\Phi'(\Phi^{-1}(\hat{f}(x))} \,,
\]
and thus we need to prove that for any unit direction $u$, denoting $p = \hat{f}(x)$,
\[
u \cdot \nabla \hat{f}(x) \leq \frac{1}{\sqrt{2\pi}} \exp\left( -\frac12 (\Phi^{-1}(p))^2 \right) \,.
\]
Note that the left-hand side can be written as follows (recall \eqref{eq:steinlemmalol}) 
\[
\E_{X \sim \mathcal{N}(0,I_n)} [ f(x+X) X \cdot u ] \,.
\]
We now claim that the supremum of the above quantity over all functions $f : \R^n \rightarrow [0,1]$, subject to the constraint that $\mathbb{E}[ f(x+X) ] = p$, is equal to:
\begin{equation} \label{eq:seb2}
\mathbb{E} [ (X \cdot u) \mathds1\{X \cdot u \geq - \Phi^{-1}(p) \}] = \frac{1}{\sqrt{2\pi}}  \exp\left( -\frac12 (\Phi^{-1}(p)^2) \right) \,,
\end{equation}
which would conclude the proof. 

To see why the latter claim is true, first notice that $h : x \mapsto \mathds1\{x \cdot u \geq - \Phi^{-1}(p)\}$ achieves equality.
Let us assume by contradiction that the maximizer is obtained at some function $f: \R^n \to [0, 1]$ different from $h$. Consider the set $\Omega^+$ where $h(x) > f(x)$ and $\Omega^-$ the set where $h(x) < f(x)$, and note that since both functions integrate to $p$, it must be that $\int_{\Omega^+} (h-f) d\mu= \int_{\Omega^-} (f-h) d \mu$ (where $\mu$ is the Gaussian measure). Now simply consider the new function $\tilde{f} = f + (h-f) \mathds1\{\Omega^+\} - (f-h) \mathds1\{\Omega^-\}$. Note that $\tilde{f}$ takes value in $[0,1]$ and integrates to $p$. Moreover, denoting $g(x) = x \cdot u$, one has $\int f g d\mu < \int \tilde{f} g d\mu$. Indeed, by definition of $h$, one has for any $x \in \Omega_+$ and $y \in \Omega^-$ that $g(x) > g(y)$. This concludes the proof.
\end{proof}

It turns out that the smoothness property of lemma 2 naturally leads to the robustness guarantee \eqref{eq:main_bound} of \citet{cohen2019certified}. To see why, let $\hat f_i: \R^n \to [0, 1]$ be the output of the smoothed classifier mapping a point $x\in\R^n$ to the probability  of it belonging to class $c_i$. Assume that the smooth classifier assigns to $x$ the class $c_A$ with probability $p_A= \hat{f}_A(x)$. Denote by $c_B$ any other class such that $c_B \neq c_A$ and $p_B = \hat{f}_B(x) \le p_A$. By lemma~2, we know that under any perturbation $\delta \in \R^n$ of $x$, 
\begin{equation}
    \Phi^{-1}\left(\hat{f}_A(x)\right) - \Phi^{-1}\left(\hat{f}_A(x + \delta)\right) \leq \|\delta\|_2.
\end{equation}
For an adversarial $\delta$, $\hat{f}_A(x + \delta) \le \hat{f}_B(x + \delta)$ for some class $c_B$, leading to
\begin{equation}\label{eq:ca}
    \Phi^{-1}\left(\hat{f}_A(x)\right) - \Phi^{-1}\left(\hat{f}_B(x + \delta)\right) \le \|\delta\|_2 .
\end{equation}

By lemma~2 applied to $\hat{f}_B$, and noting that $\hat{f}_B(x+\delta) \ge \hat{f}_B(x)$ , we know that,
\begin{equation}\label{eq:cb}
    \Phi^{-1}\left(\hat{f}_B(x + \delta)\right) - \Phi^{-1}\left(\hat{f}_B(x)\right) \leq \|\delta\|_2.
\end{equation}

Combining  \eqref{eq:ca} and \eqref{eq:cb}, it is straightforward to see that      
\begin{equation}
    \|\delta\|_2 \ge \frac12\left( \Phi^{-1}\left(p_A\right) - \Phi^{-1}\left(p_B\right)\right)
\end{equation}
The above equation gives a lower bound on the minimum $\ell_2$ adversarial perturbation required to flip the classification from $c_A$ to $c_B$. This lower bound is minimized when $p_B$ is maximized over the set of classes $C\setminus\{c_A\}$. Therefore, $c_B$ is the runner up class returned by the smoothed classifier at $x$.
Finally, the factor $\sigma$ that appears in \eqref{eq:main_bound} can be obtained by re-deriving the above with $\hat{f}(x) = \left( f * \normal(0,\sigma^2 I) \right) (x)$ and $\Phi(a)=\frac{1}{\sqrt{2 \pi}} \int_{-\infty}^a \exp\left( - \frac12 (\frac{s}{\sigma})^2 \right) ds$.

Note that both lemmas presented in this appendix give the same robustness guarantee for small gaps ($p_A - p_B$), but the second lemma is much better for large gaps (in fact, in the limit of a gap going to $1$, the second lemma gives an infinite radius while the first lemma only gives a radius of $\frac12 \sqrt{\frac{\pi}{2}}$).

\section{Another perspective for deriving \smoothAdv{}}
\label{sec:alt-deriv}
In this section we provide an alternative motivation for the \smoothAdv{} objective presented in Section~\ref{sec:smooth-adv-primary-derivation}.
We assume that we have a hard classifier $f: \R^d \to \mathcal{Y}$ which takes the form $f(x) = \argmax_{y \in \mathcal{Y}} L(x)_y$, for some function $L: \R^d \to \R^{\mathcal{Y}}$.
If $f$ is a neural network classifier, this $L$ can be taken for instance to be the map from the input to the logit layer immediately preceding the softmax.
If $f$ is of this form, then the smoothed soft classifier $g$ with parameter $\sigma^2$ associated to (the one-hot encoding of) $f$ can be written has
\begin{align}
g(x)_{y} &= \Pr_{\delta \sim \normal(0, \sigma^2 I)} \left[ \argmax_{y' \in \mathcal{Y}} L(x + \delta)_{y'} = y\right] \nonumber \\
&= \E_{\delta \sim \normal(0, \sigma^2 I)} \left[ \nu(L(x + \delta))_{y} \right] \; , \label{eq:exact-max}
\end{align}
for all $y \in \mathcal{Y}$, where $\nu: \R^d \to \R^{\mathcal{Y}}$ is the function, which at input $z$, has $y$-th coordinate equal to $1$ if and only if $y = \argmax_{y' \in \mathcal{Y}} z_{y'}$, and zero otherwise.
The function $\nu$ is somewhat hard to work with, therefore we will approximate it with a smooth function, namely, the softmax function.
Recall that the softmax function with inverse temperature parameter $\beta$ is the function $\zeta_\beta: \R^{\mathcal{Y}} \to P(\mathcal{Y})$ given by $\zeta_\beta (z)_y = e^{\beta z_y} / \sum_{y' \in \mathcal{Y}} e^{\beta z_{y'}}$.
Observe that for any $z \in \R^{\mathcal{Y}}$, we have that $\zeta_\beta (z) \to \nu(z)$ as $\beta \to \infty$.
Thus we can approximate~\eqref{eq:exact-max} with
\begin{equation}\label{eq:approx-max}
    g(x)_{y} \approx \E_{\delta \sim \normal(0, \sigma^2 I)} \left[ \zeta_\beta (L(x + \delta))_{y} \right] \; .
\end{equation}
To find an adversarial perturbation of $g$ at data point $(x, y)$, it is sufficient to find a perturbation $\hat x$ so that $g(x)_{y}$ is minimized.
Combining this with the approximation~\eqref{eq:approx-max}, we find that a heuristic to find an adversarial example for the smoothed classifier at $(x, y)$ is to solve the following optimization problem:
\begin{equation}
\label{eq:smooth-adv-1}
    \hat x = \argmin_{\| x' - x\|_2 \leq \eps} \E_{\delta \sim \normal(0, \sigma^2 I)} \left[ \zeta_\beta (L(x' + \delta))_y \right] \; ,
\end{equation}
and as we let $\beta \to \infty$, this converges to finding an adversarial example for the true smoothed classifier.

To conclude, we simply observe that for neural networks, $ \zeta_\beta (L(x + \delta))_y$ is exactly the soft classifier that is thresholded to form the hard classifier, if $\beta$ is taken to be $1$.
Therefore the solution to~\eqref{eq:smooth-adv} and~\eqref{eq:smooth-adv-1} with $\beta = 1$ are the same, since $\log$ is a monotonic function.

An interesting direction is to investigate whether varying $\beta$ in~\eqref{eq:smooth-adv-1} allows us to improve our adversarial attacks, and if they do, whether this gives us stronger adversarial training as well.
Intuitively, as we take $\beta \to \infty$, the quality of the optimal solution should increase, but the optimization problem becomes increasingly ill-behaved, and so it is not clear if the actual solution we obtain to this problem via first order methods becomes better or not.

\section{Additional Experiments}

\subsection{Adversarial attacking the base model instead of the smoothed model}\label{appendix:attack-base-model}

We compare \smoothAdv{}-ersarial training (training the smoothed classifier $g$) to:
\begin{enumerate}
    \item using vanilla adversarial training (PGD) to find adversarial examples of the \textit{base classifier} $f$ and train on them. We refer to this as \textbf{Vanilla PGD} training. 
    \item using vanilla adversarial training (PGD) to find adversarial examples of the \textit{base classifier} $f$, add Gaussian noise to them, then train on the resulting inputs. We refer to this as \textbf{Vanilla PGD+noise} training.
\end{enumerate}
For our method and the above two methods, we use $T = 2$ steps of attack, $m_{train} = 1$, and we train for $\eps \in \{0.25, 0.5, 1.0, 2.0\}$, and for $\sigma \in \{0.12, 0.25, 0.5, 1.0\}$.

Fig.~\ref{fig:ours_cohen_base} plots the best certified accuracies over all $\eps$ and $\sigma$ values, for each $\ell_2$ radius $r$ using our \smoothAdvPGD{} trained classifiers  vs. smoothed models trained via Vanilla PGD or Vanilla PGD+noise. Fig.~\ref{fig:ours_cohen_base} also plots \citet{cohen2019certified} results as a baseline. Observe that \smoothAdv-ersially trained models are more robust overall.

\begin{figure}[h]
\centering
\includegraphics[width=0.6\textwidth]{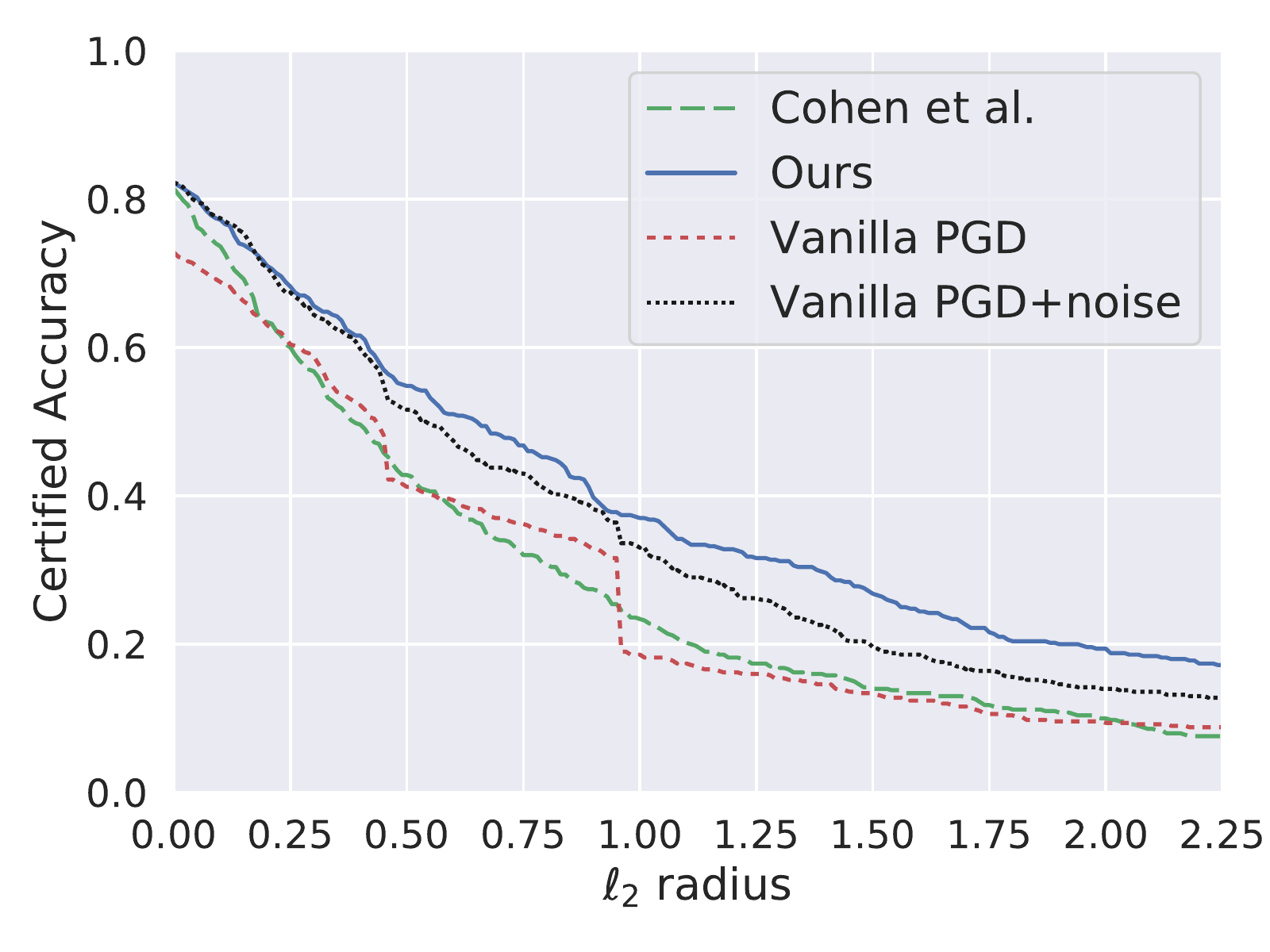}
\caption{Certified defenses: ours vs. \citet{cohen2019certified} vs. vanilla PGD vs. vanilla PGD + noise.}
\label{fig:ours_cohen_base}
\end{figure}

\subsection{Effect of number of noise samples $m_{train}$ in \eqref{eq:plug-in}  during \smoothAdv-ersarial training on the certified accuracy of smoothed classifiers}\label{appendix:effect_of_m_on_training}

As presented in Section~\ref{sec:adv-attack-results}, more noise samples $\delta_i$ lead to stronger \smoothAdv-eraial attack. Here, we demonstrate that if we train with such improved attacks, we get higher certified accuracies of the smoothed classifier. Fig.~\ref{fig:certified_ours_vary_m} plots the best certified accuracies over models trained using \smoothAdvPGD{} or \smoothAdvDDN{} with $T \in \{2,4,6,8,10\}$, $\sigma \in \{0.12, 0.25, 0.5, 1.0\}$, $\eps \in \{0.25, 0.5, 1.0, 2.0\}$, and across various number of noise samples $m_{train}$ for the attack.
Observe that models trained with higher $m_{train}$ tend to have higher certified accuracies.

\begin{figure}[h]
\centering
\includegraphics[width=0.6\textwidth]{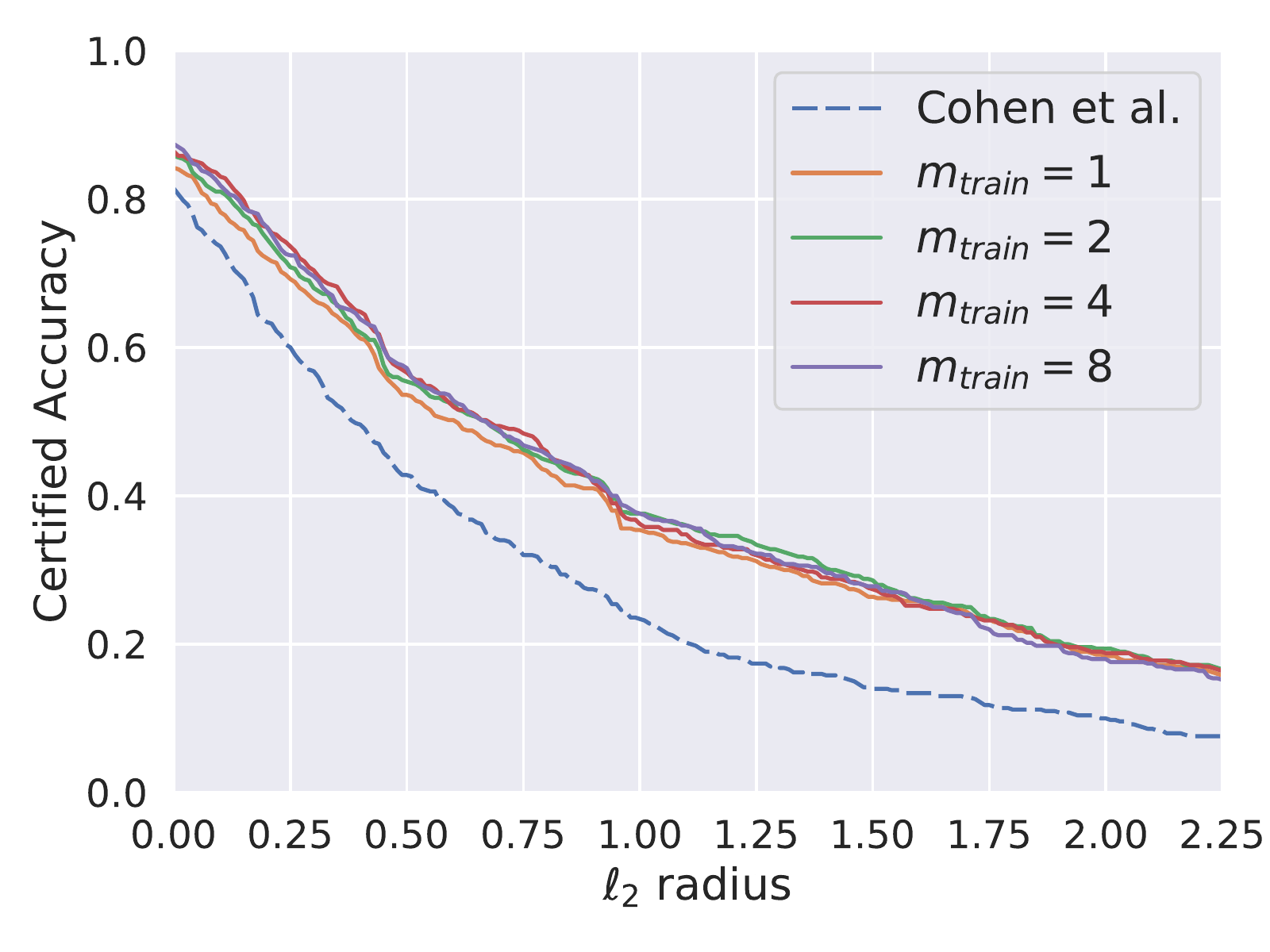}
\caption{Vary number of samples $m_{train}$.}
\label{fig:certified_ours_vary_m}
\end{figure}

\subsection{Effect of $\eps$ during training on the certified accuracy of smoothed classifiers}
\label{appendix:effect_of_eps_on_training}

Here, we analyze the effect of the maximum allowed $\ell_2$ perturbation of \smoothAdv{} during adversarial training on the robustness of the obtained smoothed classifier. Fig.~\ref{fig:vary_eps} plots the best certified accuracies for $\eps \in \{0.25, 0.5, 1.0, 2.0\}$ over models trained using \smoothAdvPGD with $T \in \{2, 4, 6, 8, 10\}$, $m_{train} \in \{1,2,4,8\}$, and $\sigma \in \{0.12, 0.25, 0.5, 1.0\}$. Observe that as  $\eps$ increases, the certified accuracies for small $\ell_2$ radii decrease, but those for large $\ell_2$ radii increase, which is expected.

\begin{figure}[h]
\centering
\includegraphics[width=0.6\textwidth]{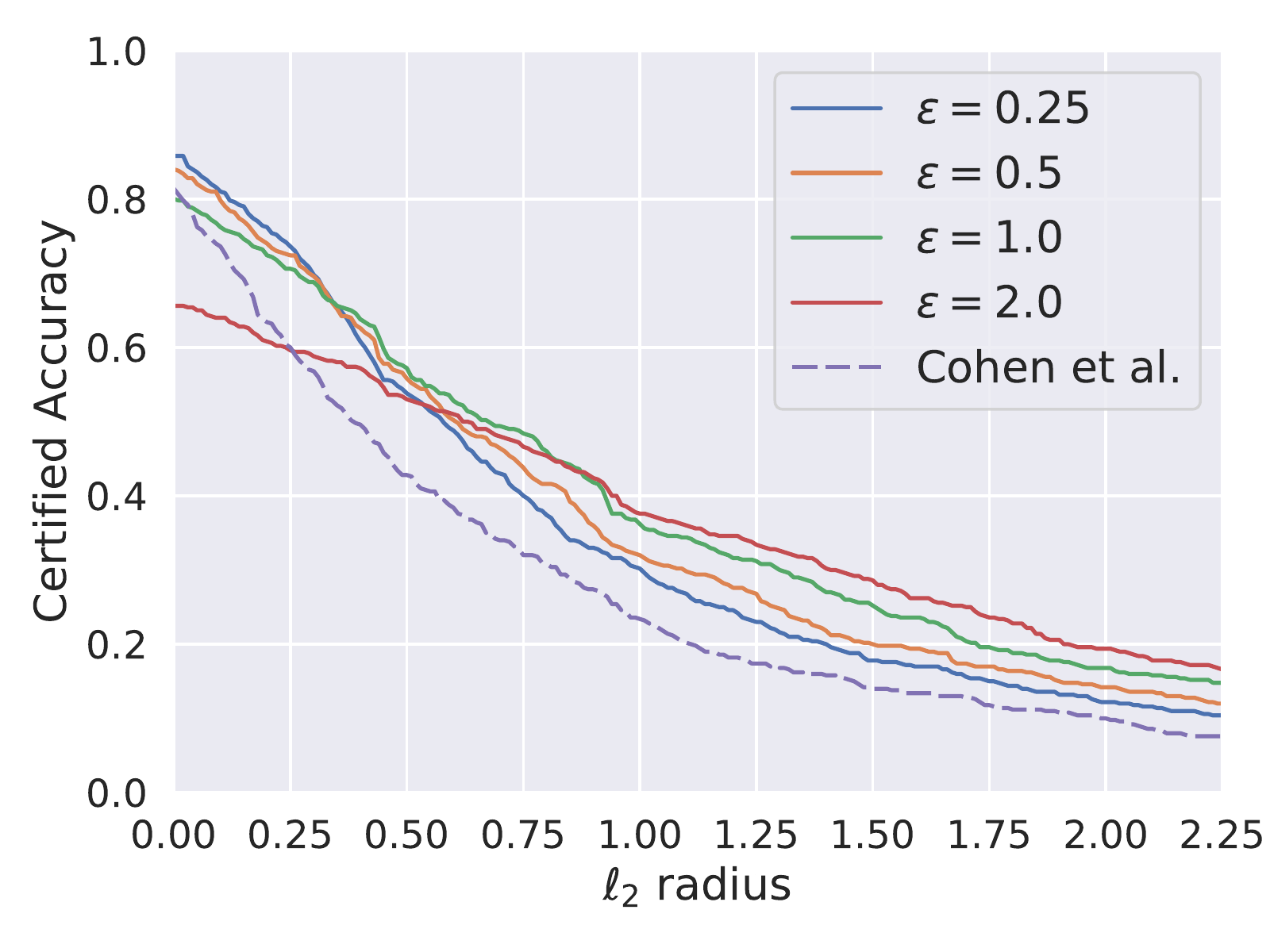}        
\caption{Vary $\eps$. Observe that as  $\eps$ increases, the certified accuracies for small $\ell_2$ radii decrease, but those for large $\ell_2$ radii increase, which is expected.}
\label{fig:vary_eps}
\end{figure}

\subsection{Effect of the number of samples $m_{test}$ in \eqref{eq:plug-in} during \smoothAdv{} attack on the empirical accuracies}
\label{appendix:effect_of_m_on_attack}

\smoothAdvPGD{} requires the evaluation of \eqref{eq:plug-in} as discussed in Section~\ref{sec:first-order}. Here, we analyze how sensitive our attack is to the number of samples $m_{test}$ used in \eqref{eq:plug-in}. Fig.~\ref{fig:empirical_cohen_vary_m_large}  shows the empirical accuracies for various values of $m_{test}$. Lower accuracies correspond to stronger attacks. For $m_{test}=1$, the vanilla PGD attack (attacking the base classifier instead of the smooth classifier) performs better than \smoothAdv{}, but as $m_{test}$ increases, our attack becomes stronger, decreasing the gap between certified and empirical accuracies. We did not observe any noticeable improvement beyond $m_{test}=128$.

\begin{figure*}[t]
\begin{center}
\includegraphics[width=0.7\textwidth]{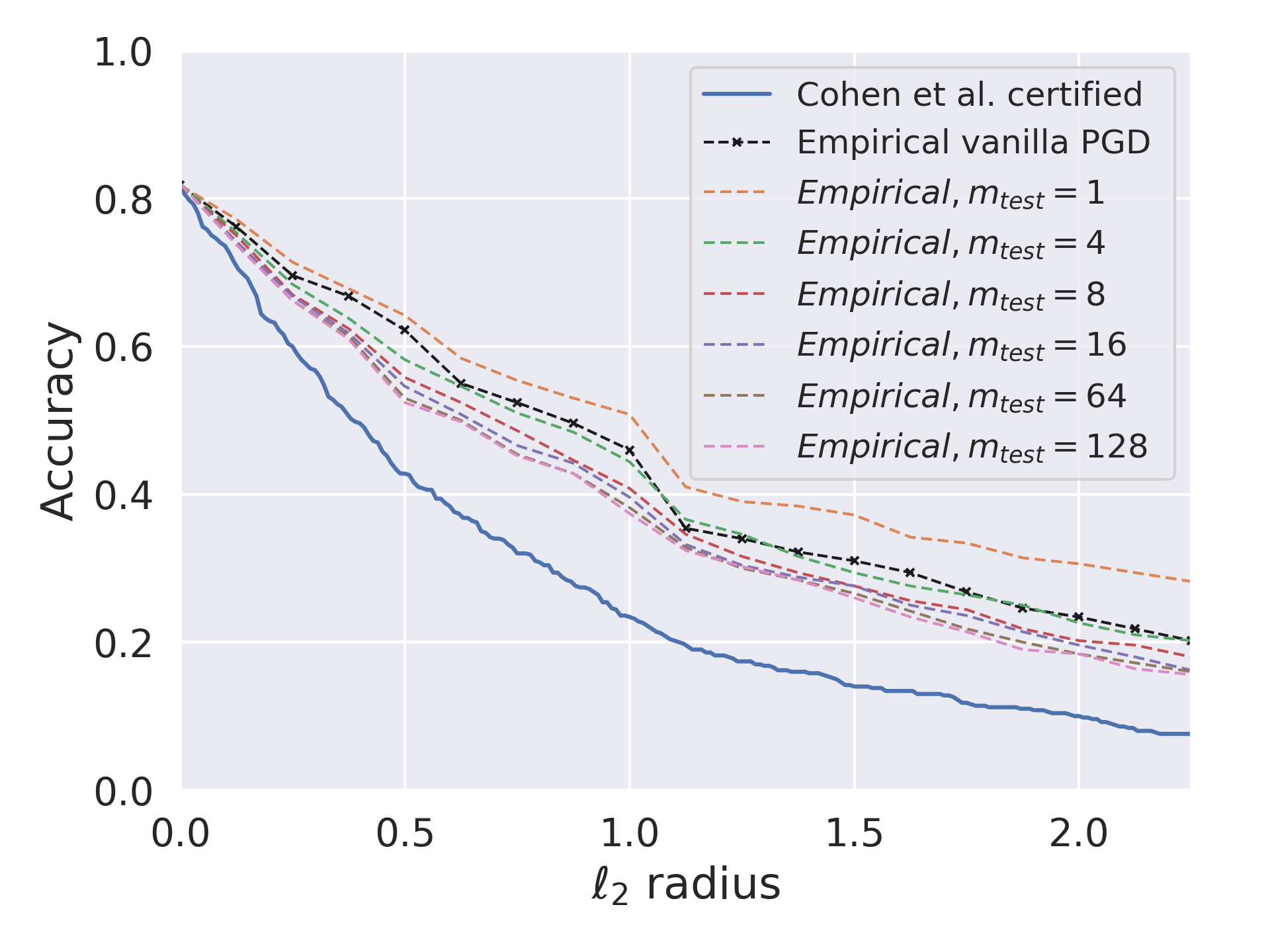}
\caption{\textbf{(A larger version of Fig.~\ref{fig:empirical_cohen_vary_m})} Certified and empirical robust accuracy of \citet{cohen2019certified}'s models on CIFAR-10. For each $\ell_2$ radius $r$, the certified/empirical accuracy is the maximum over randomized smoothing models trained using $\sigma \in \{0.12, 0.25, 0.5, 1.0\}$. The empirical accuracies are found using 20 steps of \smoothAdvPGD{}. The closer an empirical curve is to the certified curve, the stronger the corresponding attack is (the lower the better).}
\label{fig:empirical_cohen_vary_m_large}
\end{center}
\end{figure*}

\subsection{Effect of the number of Monte Carlo samples $n$ in \textsc{Predict} on the empirical accuracies}
\label{appendix:effect_of_n_on_predict}

Fig.~\ref{fig:empirical_cohen_vary_n} plots the empirical accuracies of $g$ using a \smoothAdvPGD{} attack (with $m_{test}=128$) across different numbers of Monte Carlo samples n that are used by \textsc{Predict}. Observe that the empirical accuracies increase as $n$ increases since the prediction quality of the smoothed classifier improves i.e. less predictions are abstained.

\begin{figure}[h]
\centering
\includegraphics[width=0.6\textwidth]{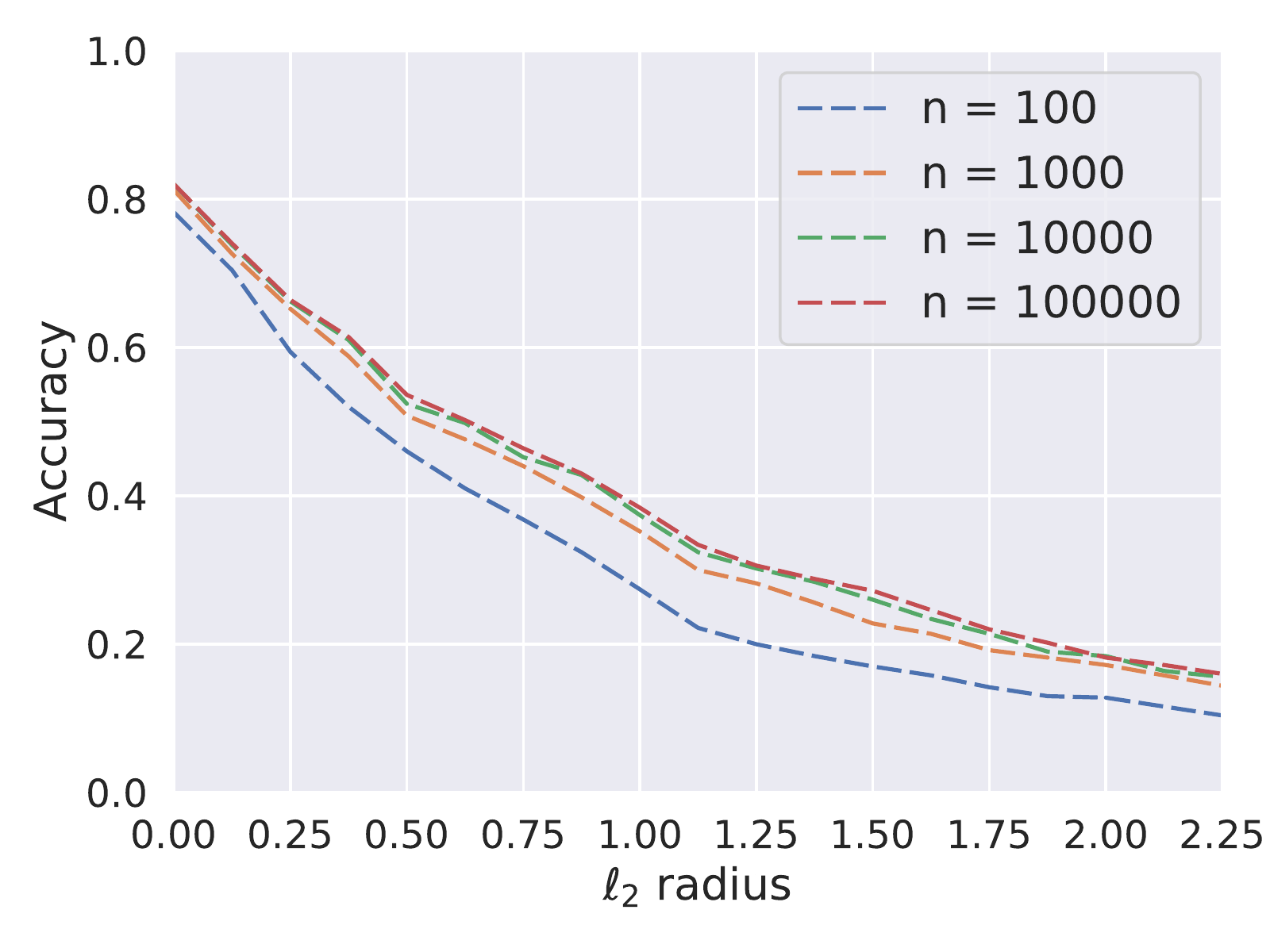}
\caption{Empirical accuracies. Vary number of samples $n$. The higher the better.}
\label{fig:empirical_cohen_vary_n}
\end{figure}

\subsection{Performance of the gradient-free estimator \eqref{eq:grad-free}}\label{appendix:grad-free}

Despite the appealing features of the gradient-free estimator \eqref{eq:grad-free} presented in Section~\ref{sec:gradient-free} as an alternative to \eqref{eq:plug-in}, in practice we find that this attack is quite weak. This is shown in Fig.~\ref{fig:gradient-free} for various values of $m_{test}$.

We speculate that this is because the variance of the gradient estimator is too high. We believe that investigating this attack in practice is an interesting direction for future work.

\begin{figure}[h]
\centering
\includegraphics[width=1.0\textwidth]{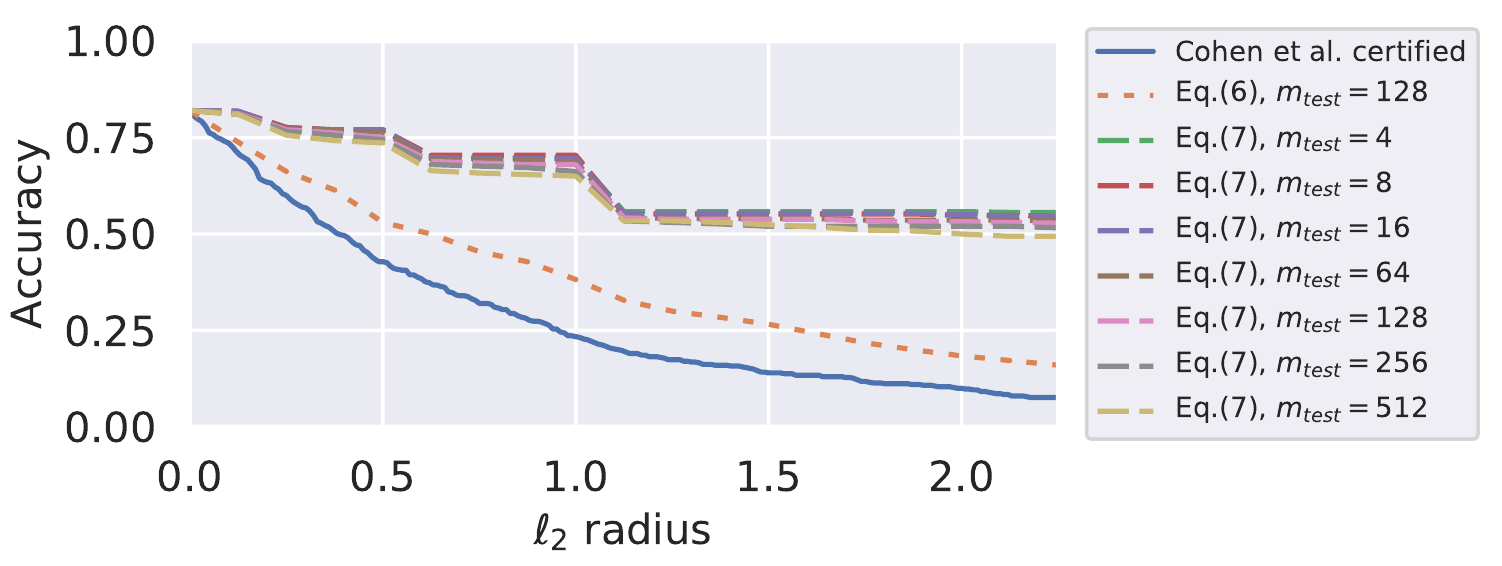}
\caption{The emprirical accuracies found by the attack \eqref{eq:smooth-adv} using the plug-in estimator \eqref{eq:plug-in} vs. the gradient-free estimator \eqref{eq:grad-free}. The closer an empirical curve is to the certified curve, the stronger the attack.}
\label{fig:gradient-free}
\end{figure}

\subsection{Certification Abstention Rate}
In this section, we compare the certification abstention rates of our smoothed models against those of \citet{cohen2019certified}'s models.
Table~\ref{table:abstention_rate} reports the abstention rates for the best models at various $\ell_2$-radii. These are the models corresponding to Table~\ref{cifar-certified-accuracy}.
Our models have a substantially lower abstention rate across all $\ell_2$ radii.

Note that \citet{cohen2019certified} reported the abstention rates for \emph{prediction} (but not certification), which tend to be lower than the certification abstention rates.

\begin{table}[t]
\caption{The certification abstention rate of our best CIFAR-10 classifiers at various $\ell_2$ radii.}
\label{table:abstention_rate}
\begin{center}
\begin{sc}
\begin{tabular}{l | c c c c c c c c c}
\toprule
$\ell_2$ Radius (CIFAR-10)& $0.25$& $0.5$& $0.75$& $1.0$& $1.25$& $1.5$& $1.75$& $2.0$& $2.25$\\
\midrule
\citet{cohen2019certified} (\%) &8.0 & 15.2 & 15.2 & 29.5 & 29.5 & 29.5 & 29.5 & 29.5 & 29.5\\
Ours (\%)& 1.7 & 3.3 & 3.0 & 5.8 & 4.7 & 4.7 & 4.3 & 13.2 & 13.2\\
+ Pre-training (\%) & 0.9 & 2.9 & 2.1 & 3.9 & 3.9 & 3.9 & 3.9 & 10.4 & 10.4\\
+ Semi-supervision (\%) & 1.1 & 2.8 & 2.8 & 6.3 & 6.3 & 4.4 & 4.3 & 11.5 & 11.5\\
+ Both(\%)  & 1.0 & 2.7 & 2.1 & 4.1 & 4.1 & 4.1 & 4.1 & 11.6 & 11.6\\
\bottomrule
\end{tabular}
\end{sc}
\end{center}
\end{table}

\section{Experiments Details}\label{appendix:experiments-details}
Here we include details of all the experiments conducted in this paper.

\paragraph{Attacks used in the paper} 
We use two of the strongest attacks in the literature, projected gradient descent (PGD) \cite{madry2017towards} and decoupled direction and norm (DDN) \cite{rony2018decoupling} attacks. We adapt these attacks such that their gradient steps are given by \eqref{eq:plug-in}, and we call the resulting attacks \smoothAdvPGD{} and \smoothAdvDDN{}, respectively.

For PGD (\smoothAdvPGD{}), we use a constant step size $\gamma = 2\frac{\eps}{T}$ where $T$ is the number of attack steps, and $\eps$ is the maximum allowed $\ell_2$ perturbation of the input.  

For DDN (\smoothAdvDDN{}), the attack objective is in fact different than that of PGD (i.e. different that \eqref{eq:smooth-adv}). DDN tries to find the ``closest'' adversarial example to the input instead of finding the ``best'' adversarial example (in terms of maximizing the loss in a given neighborhood of the input). We stick to the hyperparameters used in the original paper \cite{rony2018decoupling}. We  use $\eps_0 = 1$, $\gamma = 0.05$, and an initial step size $\alpha = 1$ that is reduced with cosine annealing to 0.01 in the last iteration (see \cite{rony2018decoupling} for the definition of these parameters). We experimented with very few iterations ($\{2, 4, 6, 8, 10\}$) as compared to the original paper, but we still got good results.

We emphasize that we are not using PGD and DDN to attack the \textit{base classifer} $f$ of a smoothed model, instead we are using them to adversarially train \textit{smoothed classiers} (see Pseudocode~\ref{pseudocode-smoothadv}).

\paragraph{Training details}
In order to report certified radii in the original coordinates, we first added Gaussian noise and/or do adversarial attacks, and then standardized the data (in contrast to importing a standardized dataset).
Specifically, in our PyTorch implementation, the first layer of the base classifier is a normalization layer that performed a channel-wise standardization of its input.

For both ImageNet and CIFAR-10, we trained the base classifier with random horizontal flips and random crops (in addition to the Gaussian data augmentation discussed in Section~\ref{sec:adv-training}).

The main training algorithm is shown in Pseudocode~\ref{pseudocode-smoothadv}. It has the following parameters: $B$ is the mini-batch size, $m$ is the number of noise samples used for gradient estimation in~(\ref{eq:plug-in}) as  well as for Gaussian noise data augmentation, and $T$ is the number of steps of an attack.

We point out few remarks. 
\begin{enumerate}
    \item First, an important parameter is the radius of the attack $\eps$. During the first epoch, it is set to zero, then we linearly increase it over the first ten epochs, then it stays constant.
    \item Second, we are reusing the same noise samples during every step of our attack as well as augmentation. Intuitively, it helps to stabilize the attack process.
    \item Finally, the way training is described in Pseudocode~\ref{pseudocode-smoothadv} is not efficient; it needs to be appropriately batched so that we compute adversarial examples for every input in a batch at the same time.
\end{enumerate}

\paragraph{Compute details and training time}
On CIFAR-10, we trained using SGD on one NVIDIA P100 GPU. We train for 150 epochs. We use a batch size of 256, and an initial learning rate of 0.1 which drops by a factor of 10 every 50 epochs. Training time varies between few hours to few days, depending on how many attack steps $T$ and noise samples $m$ are used in Pseudocode~\ref{pseudocode-smoothadv}.

On ImageNet we trained with synchronous SGD on four NVIDIA V100 GPUs. We train for 90 epochs. We use a batch size of 400, and an initial learning rate of 0.1 which drops by a factor of 10 every 30 epochs. Training time varies between 2 to 6 days depending on whether we are doing \smoothAdv-ersarial training or just Gaussian noise training (similar to \citet{cohen2019certified}).

\paragraph{Models used}
The models used in this paper are similar to those used in \citet{cohen2019certified}: a ResNet-50 \citep{he2016deep} on ImageNet, and ResNet-110 on CIFAR-10. These models can be found on the github repo accompanying \cite{cohen2019certified} \url{https://github.com/locuslab/smoothing/blob/master/code/architectures.py}.

\paragraph{Parameters of \textsc{Certify} amd \textsc{Predict}}
For details of these algorithms, please see the \textit{Pseudocode} in \cite{cohen2019certified}.

For \textsc{Certify}, unless otherwise specified, we use $n=100,000$, $n_0=100$, $\alpha = 0.001$.

For \textsc{Predict}, unless otherwise specified, we use $n=100,000$ and $\alpha =0.001$.

\paragraph{Source code}
Our code and trained models are publicly available at \url{http://github.com/Hadisalman/smoothing-adversarial}.
The repository also includes all our training/certification logs, which enables the replication of all the results of this paper by running a single piece of code. Check the repository for more details.

\section{Details for Pre-training and Semi-supervision to Improve the Provable Robustness}

\subsection{Pre-training}\label{appendix:pretraining}
In this appendix, we describe the details of how we employ pre-training within our framework to boost the certified robustness of our models. We pretrain smoothed classifiers on a 32x32 down-sampled version of ImageNet (\href{http://www.image-net.org/small/download.php}{ImageNet32}) as done by \citet{hendrycks2019using}. Then  we fine-tune all the weights of these  models on CIFAR-10 (with the 1000-dimensional logit layer of each model replaced by a randomly initialized 10-dimensional logit layer suitable for CIFAR-10).

\paragraph{ImageNet32 training} We train ResNet-110 architectures on ImageNet32 using SGD on one NVIDIA P100 GPU. We train for 150 epochs. We use a batch size of 256, and an initial learning rate of 0.1 which drops by a factor of 10 every 50 epochs. We use  \smoothAdvPGD{} with $T = 2$ steps and  $m_{train}=1$ noise samples. We train a total of 16 models each corresponding to a choice of $\sigma \in \{0.12, 0.25, 0.5, 1.0\}$ and $\eps \in \{0.25, 0.5, 1.0, 2.0 \}$.

\paragraph{Fine-tuning on CIFAR-10} For each choice of $\sigma$ and $\eps$, we fine tune the corresponding ImageNet32 model on CIFAR-10; we replace the 1000-dimensional logit layer of each model with a randomly initialized 10-dimensional logit layer suitable for CIFAR-10, then we train for 30 epochs with a constant learning rate of 0.001 and a batch size of 256. We use  \smoothAdvPGD{} with $T \in \{ 2,4,6,8,10\}$ and  $m_{train} \in \{1,2,4,8\}$.

\subsection{Semi-supervised Learning}\label{appendix:semisupervision}
In this appendix, we detail how we employ semi-supervised learning \cite{carmon2019unlabeled} within our framework to boost the certified robustness of our models.

We train our CIFAR-10 smoothed classifiers via the \emph{self-training} technique of \cite{carmon2019unlabeled} using their 500K unlabelled dataset.
We equip this dataset with pseudo-labels generated by a standard neural network trained on CIFAR-10, as in \cite{carmon2019unlabeled}; see \cite{carmon2019unlabeled} for more details\footnote{The 500K unlabelled dataset was not public at the time this paper was written. We obtained it, along with the pseudo-labels, from the authors of \cite{carmon2019unlabeled}. We refer the reader to the authors of \cite{carmon2019unlabeled} to obtain this dataset if interested in replicating our self-training results. }.

Self-training a smoothed classifier works as follows:
at every step we randomly sample either a labelled minibatch from CIFAR-10, or a pseudo-labelled minibatch from the 500K dataset:
\begin{enumerate}
    \item for a labelled minibatch, we follow Pseudocode~\ref{pseudocode-smoothadv} as is.
    \item for a pseudo-labelled minibatch, we scale the CE loss by a factor of $\eta \in \{0.1, 0.5, 1.0\}$ and we follow the rest of Pseudocode~\ref{pseudocode-smoothadv}.
\end{enumerate}

We use \smoothAdvPGD{} with $T \in \{ 2,4,6,8,10\}$, $m_{train} = 1 $, $\sigma \in \{0.12, 0.25, 0.5, 1.0\}$, and $\eps \in \{0.25, 0.5, 1.0, 2.0 \}$.

\subsection{Semi-supervised Learning with Pre-training}\label{appendix:semisupervision_and_pretraining}
We also experiment with combining semi-supervised learning with pre-training in the hopes of obtaining further improvements.
We start from the same ResNet-110 models pretrained on  ImageNet32 as in Appendix~\ref{appendix:pretraining}. Then we finetune these models using semi-supervision, as in Appendix~\ref{appendix:semisupervision}, for 30 epochs with a learning rate of 0.001.
We use \smoothAdvPGD{} with $T \in \{ 2,4,6,8,10\}$, $m_{train} = 1 $, $\sigma \in \{0.12, 0.25, 0.5, 1.0\}$, and $\eps \in \{0.25, 0.5, 1.0, 2.0 \}$.

\section{$\ell_{2}$ to $\ell_{\infty}$ Certified Defense on ImageNet}\label{appendix:l2_to_linf}
We find our $\ell_2$-robust ImageNet models enjoy non-trivial $\ell_{\infty}$ certified robustness. In~Table~\ref{table:certified_linf_imagenet}, we report the best $\ell_{\infty}$ certified accuracy that we get at a radius of 1/255 (implied by the $\ell_{2}$ certified accuracy at a radius of $1.5 \approx \sqrt{3\times 224^2} / 255$).
We exceed previous state-of-the-art in certified $\ell_{\infty}$ defenses by around $8.2\%$.
\begin{table}[h]
\caption{Certified $\ell_{\infty}$ robustness at a radius of $\frac{1}{255}$ on ImageNet.}
\label{table:certified_linf_imagenet}
\begin{center}
\begin{sc}
\begin{tabular}{l|cc}
Model                 & $\ell_{\infty}$ Acc. at $1/255$ & Standard Acc. \\
\midrule
Ours (\%)                        &           \textbf{38.2}                &   54.6    \\
\citet{cohen2019certified} (\%) &            28.6              &     \textbf{57.2}
\end{tabular}
\end{sc}
\end{center}
\end{table}

\clearpage
\section{ImageNet and CIFAR-10 Detailed Results}
\label{appendix:detialed_certification_results}
In this appendix, we include the certified accuracies of each mode that we use in the paper. For each $\ell_2$ radius, we highlight the best accuracy across all models. Note that we outperform the models of \citet{cohen2019certified} (first three rows of each table) over all $\ell_2$ radii by wide margins.
\begin{table}[h]
\caption{Approximate certified test accuracy on ImageNet.  Each row is a setting of the hyperparameters $\sigma$ and $\epsilon$, each column is an $\ell_2$ radius.  The entry of the best $\sigma$ for each radius is bolded. For comparison, random guessing would attain 0.001 accuracy.}
\label{table:imagenet_certify_results}
\begin{center}
\begin{small}
\begin{sc}
\begin{tabular}{c  c | c c c c c c c c c c} 
\toprule
&$\ell_2$ Radius (ImageNet)& $0.0$& $0.5$& $1.0$& $1.5$& $2.0$& $2.5$& $3.0$& $3.5$\\
\midrule
\multirow{3}{*}{\rotatebox{0}{\cite{cohen2019certified}}}
&$\sigma = 0.25$ & \textbf{0.67} & 0.49 & 0.00 & 0.00 & 0.00 & 0.00 & 0.00 & 0.00\\
&$\sigma = 0.50$ & 0.57 & 0.46 & 0.37 & 0.29 & 0.00 & 0.00 & 0.00 & 0.00\\
&$\sigma = 1.00$ & 0.44 & 0.38 & 0.33 & 0.26 & 0.19 & 0.15 & 0.12 & 0.09\\
\midrule
\midrule
&$\sigma = 0.25$\quad \quad $\epsilon = 0.5$ & 0.63 & 0.54 & 0.00 & 0.00 & 0.00 & 0.00 & 0.00 & 0.00\\
&$\sigma = 0.25$\quad \quad $\epsilon = 1.0$ & 0.62 & 0.54 & 0.00 & 0.00 & 0.00 & 0.00 & 0.00 & 0.00\\
\multirow{3}{*}{\rotatebox{90}{\smoothAdvPGD}}
&$\sigma = 0.25$\quad \quad $\epsilon = 2.0$ & 0.56 & 0.52 & 0.00 & 0.00 & 0.00 & 0.00 & 0.00 & 0.00\\
&$\sigma = 0.25$\quad \quad $\epsilon = 4.0$ & 0.49 & 0.45 & 0.00 & 0.00 & 0.00 & 0.00 & 0.00 & 0.00\\
\bottomrule
\end{tabular}
\end{sc}
\end{small}
\end{center}
\end{table}

\begin{table}[h]
\caption{\smoothAdv{}-ersarial training $T=2$ steps, $m_{train} = 1$ sample.}
\label{table:cifar_certify_results_PGD_DDN_2steps_1samples}
\begin{center}
\begin{small}
\begin{sc}
\begin{tabular}{cc | c c c c c c c c c c}
\toprule
&$\ell_2$ Radius (CIFAR-10) & $0.0$& $0.25$& $0.5$& $0.75$& $1.0$& $1.25$& $1.5$& $1.75$& $2.0$& $2.25$\\
\midrule
\multirow{3}{*}{\rotatebox{0}{\cite{cohen2019certified}}}
&$\sigma = 0.12$ & 0.81 & 0.59 & 0.00 & 0.00 & 0.00 & 0.00 & 0.00 & 0.00 & 0.00 & 0.00\\
&$\sigma = 0.25$ & 0.75 & 0.60 & 0.43 & 0.27 & 0.00 & 0.00 & 0.00 & 0.00 & 0.00 & 0.00\\
&$\sigma = 0.50$ & 0.65 & 0.55 & 0.41 & 0.32 & 0.23 & 0.15 & 0.09 & 0.05 & 0.00 & 0.00\\
&$\sigma = 1.00$ & 0.47 & 0.39 & 0.34 & 0.28 & 0.22 & 0.17 & 0.14 & 0.12 & 0.10 & 0.08\\
\midrule
\midrule
&$\sigma = 0.12$\quad \quad $\epsilon = 0.25$ & \textbf{0.84} & \textbf{0.69} & 0.00 & 0.00 & 0.00 & 0.00 & 0.00 & 0.00 & 0.00 & 0.00\\
&$\sigma = 0.12$\quad \quad $\epsilon = 0.50$ & 0.75 & 0.63 & 0.00 & 0.00 & 0.00 & 0.00 & 0.00 & 0.00 & 0.00 & 0.00\\
&$\sigma = 0.12$\quad \quad $\epsilon = 1.00$ & 0.75 & 0.63 & 0.00 & 0.00 & 0.00 & 0.00 & 0.00 & 0.00 & 0.00 & 0.00\\
\bottomrule
\end{tabular}
\end{sc}
\end{small}
\end{center}
\end{table}

\begin{table}[h]
\caption{\smoothAdv{}-ersarial training $T=4$ steps, $m_{train} = 1$ sample.}
\label{table:cifar_certify_results_PGD_DDN_4steps_1samples}
\begin{center}
\begin{small}
\begin{sc}
\begin{tabular}{c c| c c c c c c c c c c}
\toprule
&$\ell_2$ Radius (CIFAR-10)& $0.0$& $0.25$& $0.5$& $0.75$& $1.0$& $1.25$& $1.5$& $1.75$& $2.0$& $2.25$\\
\midrule
\multirow{3}{*}{\rotatebox{0}{\cite{cohen2019certified}}}
&$\sigma = 0.12$ & 0.81 & 0.59 & 0.00 & 0.00 & 0.00 & 0.00 & 0.00 & 0.00 & 0.00 & 0.00\\
&$\sigma = 0.25$ & 0.75 & 0.60 & 0.43 & 0.27 & 0.00 & 0.00 & 0.00 & 0.00 & 0.00 & 0.00\\
&$\sigma = 0.50$ & 0.65 & 0.55 & 0.41 & 0.32 & 0.23 & 0.15 & 0.09 & 0.05 & 0.00 & 0.00\\
&$\sigma = 1.00$ & 0.47 & 0.39 & 0.34 & 0.28 & 0.22 & 0.17 & 0.14 & 0.12 & 0.10 & 0.08\\
\midrule
\midrule
&$\sigma = 0.12$\quad \quad $\epsilon = 0.25$ & 0.82 & 0.68 & 0.00 & 0.00 & 0.00 & 0.00 & 0.00 & 0.00 & 0.00 & 0.00\\
&$\sigma = 0.12$\quad \quad $\epsilon = 0.50$ & 0.80 & 0.67 & 0.00 & 0.00 & 0.00 & 0.00 & 0.00 & 0.00 & 0.00 & 0.00\\
&$\sigma = 0.12$\quad \quad $\epsilon = 1.00$ & 0.78 & \textbf{0.69} & 0.00 & 0.00 & 0.00 & 0.00 & 0.00 & 0.00 & 0.00 & 0.00\\
\bottomrule
\end{tabular}
\end{sc}
\end{small}
\end{center}
\end{table}

\begin{table}[h]
\caption{\smoothAdv{}-ersarial training $T=6$ steps, $m_{train} = 1$ sample.}
\label{table:cifar_certify_results_PGD_DDN_6steps_1samples}
\begin{center}
\begin{small}
\begin{sc}
\begin{tabular}{c c| c c c c c c c c c c}
\toprule
&$\ell_2$ Radius (CIFAR-10)& $0.0$& $0.25$& $0.5$& $0.75$& $1.0$& $1.25$& $1.5$& $1.75$& $2.0$& $2.25$\\
\midrule
\multirow{3}{*}{\rotatebox{0}{\cite{cohen2019certified}}}
&$\sigma = 0.12$ & \textbf{0.81} & 0.59 & 0.00 & 0.00 & 0.00 & 0.00 & 0.00 & 0.00 & 0.00 & 0.00\\
&$\sigma = 0.25$ & 0.75 & 0.60 & 0.43 & 0.27 & 0.00 & 0.00 & 0.00 & 0.00 & 0.00 & 0.00\\
&$\sigma = 0.50$ & 0.65 & 0.55 & 0.41 & 0.32 & 0.23 & 0.15 & 0.09 & 0.05 & 0.00 & 0.00\\
&$\sigma = 1.00$ & 0.47 & 0.39 & 0.34 & 0.28 & 0.22 & 0.17 & 0.14 & 0.12 & 0.10 & 0.08\\
\midrule
\midrule
&$\sigma = 0.12$\quad \quad $\epsilon = 0.25$ & \textbf{0.81} & 0.66 & 0.00 & 0.00 & 0.00 & 0.00 & 0.00 & 0.00 & 0.00 & 0.00\\
&$\sigma = 0.12$\quad \quad $\epsilon = 0.50$ & \textbf{0.81} & 0.67 & 0.00 & 0.00 & 0.00 & 0.00 & 0.00 & 0.00 & 0.00 & 0.00\\
&$\sigma = 0.12$\quad \quad $\epsilon = 1.00$ & 0.76 & 0.66 & 0.00 & 0.00 & 0.00 & 0.00 & 0.00 & 0.00 & 0.00 & 0.00\\
\bottomrule
\end{tabular}
\end{sc}
\end{small}
\end{center}
\end{table}

\begin{table}[h]
\caption{\smoothAdv{}-ersarial training $T=8$ steps, $m_{train} = 1$ sample.}
\label{table:cifar_certify_results_PGD_DDN_8steps_1samples}
\begin{center}
\begin{small}
\begin{sc}
\begin{tabular}{c c | c c c c c c c c c c}
\toprule
&$\ell_2$ Radius (CIFAR-10)& $0.0$& $0.25$& $0.5$& $0.75$& $1.0$& $1.25$& $1.5$& $1.75$& $2.0$& $2.25$\\
\midrule
\multirow{3}{*}{\rotatebox{0}{\cite{cohen2019certified}}}
&$\sigma = 0.12$ & 0.81 & 0.59 & 0.00 & 0.00 & 0.00 & 0.00 & 0.00 & 0.00 & 0.00 & 0.00\\
&$\sigma = 0.25$ & 0.75 & 0.60 & 0.43 & 0.27 & 0.00 & 0.00 & 0.00 & 0.00 & 0.00 & 0.00\\
&$\sigma = 0.50$ & 0.65 & 0.55 & 0.41 & 0.32 & 0.23 & 0.15 & 0.09 & 0.05 & 0.00 & 0.00\\
&$\sigma = 1.00$ & 0.47 & 0.39 & 0.34 & 0.28 & 0.22 & 0.17 & 0.14 & 0.12 & 0.10 & 0.08\\
\midrule
\midrule
&$\sigma = 0.12$\quad \quad $\epsilon = 0.25$ & \textbf{0.83} & 0.68 & 0.00 & 0.00 & 0.00 & 0.00 & 0.00 & 0.00 & 0.00 & 0.00\\
&$\sigma = 0.12$\quad \quad $\epsilon = 0.50$ & 0.82 & \textbf{0.70} & 0.00 & 0.00 & 0.00 & 0.00 & 0.00 & 0.00 & 0.00 & 0.00\\
&$\sigma = 0.12$\quad \quad $\epsilon = 1.00$ & 0.77 & 0.66 & 0.00 & 0.00 & 0.00 & 0.00 & 0.00 & 0.00 & 0.00 & 0.00\\
\bottomrule
\end{tabular}
\end{sc}
\end{small}
\end{center}
\end{table}

\begin{table}[h]
\caption{\smoothAdv{}-ersarial training $T=10$ steps, $m_{train} = 1$ sample.}
\label{table:cifar_certify_results_PGD_DDN_10steps_1samples}
\begin{center}
\begin{small}
\begin{sc}
\begin{tabular}{c c | c c c c c c c c c c}
\toprule
&$\ell_2$ Radius (CIFAR-10)& $0.0$& $0.25$& $0.5$& $0.75$& $1.0$& $1.25$& $1.5$& $1.75$& $2.0$& $2.25$\\
\midrule
\multirow{3}{*}{\rotatebox{0}{\cite{cohen2019certified}}}
&$\sigma =0.12$ & 0.81 & 0.59 & 0.00 & 0.00 & 0.00 & 0.00 & 0.00 & 0.00 & 0.00 & 0.00\\
&$\sigma =0.25$ & 0.75 & 0.60 & 0.43 & 0.27 & 0.00 & 0.00 & 0.00 & 0.00 & 0.00 & 0.00\\
&$\sigma =0.50$ & 0.65 & 0.55 & 0.41 & 0.32 & 0.23 & 0.15 & 0.09 & 0.05 & 0.00 & 0.00\\
&$\sigma =1.00$ & 0.47 & 0.39 & 0.34 & 0.28 & 0.22 & 0.17 & 0.14 & 0.12 & 0.10 & 0.08\\
\midrule
\midrule
&$\sigma =0.12$\quad \quad $\epsilon = 0.25$ & \textbf{0.84} & 0.66 & 0.00 & 0.00 & 0.00 & 0.00 & 0.00 & 0.00 & 0.00 & 0.00\\
&$\sigma =0.12$\quad \quad $\epsilon = 0.50$ & 0.81 & 0.66 & 0.00 & 0.00 & 0.00 & 0.00 & 0.00 & 0.00 & 0.00 & 0.00\\
&$\sigma =0.12$\quad \quad $\epsilon = 1.00$ & 0.76 & 0.67 & 0.00 & 0.00 & 0.00 & 0.00 & 0.00 & 0.00 & 0.00 & 0.00\\
\bottomrule
\end{tabular}
\end{sc}
\end{small}
\end{center}
\end{table}

\begin{table}[h]
\caption{\smoothAdvDDN{} training $T=4$ steps, $m_{train} \in\{1,2,4,8\}$ samples.}
\label{table:cifar_certify_results_DDN_4steps_2/4/8samples}
\begin{center}
\begin{tiny}
\begin{sc}
\begin{tabular}{c c | c c c c c c c c c c}
\toprule
&$\ell_2$ Radius (CIFAR-10)& $0.0$& $0.25$& $0.5$& $0.75$& $1.0$& $1.25$& $1.5$& $1.75$& $2.0$& $2.25$\\
\midrule
\midrule
\multirow{3}{*}{\rotatebox{0}{\cite{cohen2019certified}}}
&$\sigma =0.12$ & 0.81 & 0.59 & 0.00 & 0.00 & 0.00 & 0.00 & 0.00 & 0.00 & 0.00 & 0.00\\
&$\sigma =0.25$ & 0.75 & 0.60 & 0.43 & 0.27 & 0.00 & 0.00 & 0.00 & 0.00 & 0.00 & 0.00\\
&$\sigma =0.50$ & 0.65 & 0.55 & 0.41 & 0.32 & 0.23 & 0.15 & 0.09 & 0.05 & 0.00 & 0.00\\
&$\sigma =1.00$ & 0.47 & 0.39 & 0.34 & 0.28 & 0.22 & 0.17 & 0.14 & 0.12 & 0.10 & 0.08\\
\midrule
&$\sigma =0.12$\quad \quad $\epsilon = 0.25$ & 0.82 & 0.68 & 0.00 & 0.00 & 0.00 & 0.00 & 0.00 & 0.00 & 0.00 & 0.00\\
&$\sigma =0.12$\quad \quad $\epsilon = 0.50$ & 0.80 & 0.67 & 0.00 & 0.00 & 0.00 & 0.00 & 0.00 & 0.00 & 0.00 & 0.00\\
&$\sigma =0.12$\quad \quad $\epsilon = 1.00$ & 0.78 & 0.69 & 0.00 & 0.00 & 0.00 & 0.00 & 0.00 & 0.00 & 0.00 & 0.00\\
\bottomrule
\end{tabular}
\end{sc}
\end{tiny}
\end{center}
\end{table}

\begin{table}[h]
\caption{\smoothAdvDDN{} training $T=10$ steps, $m_{train} \in\{1,2,4,8\}$ samples.}
\label{table:cifar_certify_results_DDN_10steps_2/4/8samples}
\begin{center}
\begin{tiny}
\begin{sc}
\begin{tabular}{c c | c c c c c c c c c c}
\toprule
&$\ell_2$ Radius (CIFAR-10)& $0.0$& $0.25$& $0.5$& $0.75$& $1.0$& $1.25$& $1.5$& $1.75$& $2.0$& $2.25$\\
\midrule
\midrule
\multirow{3}{*}{\rotatebox{0}{\cite{cohen2019certified}}}
&$\sigma =0.12$ & 0.81 & 0.59 & 0.00 & 0.00 & 0.00 & 0.00 & 0.00 & 0.00 & 0.00 & 0.00\\
&$\sigma =0.25$ & 0.75 & 0.60 & 0.43 & 0.27 & 0.00 & 0.00 & 0.00 & 0.00 & 0.00 & 0.00\\
&$\sigma =0.50$ & 0.65 & 0.55 & 0.41 & 0.32 & 0.23 & 0.15 & 0.09 & 0.05 & 0.00 & 0.00\\
&$\sigma =1.00$ & 0.47 & 0.39 & 0.34 & 0.28 & 0.22 & 0.17 & 0.14 & 0.12 & 0.10 & 0.08\\
\midrule
&$\sigma =0.12$\quad \quad $\epsilon = 0.25$ & 0.84 & 0.66 & 0.00 & 0.00 & 0.00 & 0.00 & 0.00 & 0.00 & 0.00 & 0.00\\
&$\sigma =0.12$\quad \quad $\epsilon = 0.50$ & 0.81 & 0.66 & 0.00 & 0.00 & 0.00 & 0.00 & 0.00 & 0.00 & 0.00 & 0.00\\
&$\sigma =0.12$\quad \quad $\epsilon = 1.00$ & 0.76 & 0.67 & 0.00 & 0.00 & 0.00 & 0.00 & 0.00 & 0.00 & 0.00 & 0.00\\
\bottomrule
\end{tabular}
\end{sc}
\end{tiny}
\end{center}
\end{table}

\begin{table}[h]
\caption{\smoothAdvPGD{} training $T=2$ steps, $m_{train} \in\{1,2,4,8\}$ samples.}
\label{table:cifar_certify_results_PGD_2steps_2/4/8samples}
\begin{center}
\begin{tiny}
\begin{sc}
\begin{tabular}{c c | c c c c c c c c c c}
\toprule
&$\ell_2$ Radius (CIFAR-10)& $0.0$& $0.25$& $0.5$& $0.75$& $1.0$& $1.25$& $1.5$& $1.75$& $2.0$& $2.25$\\
\midrule
\multirow{3}{*}{\rotatebox{0}{\cite{cohen2019certified}}}
&$\sigma =0.12$ & 0.81 & 0.59 & 0.00 & 0.00 & 0.00 & 0.00 & 0.00 & 0.00 & 0.00 & 0.00\\
&$\sigma =0.25$ & 0.75 & 0.60 & 0.43 & 0.27 & 0.00 & 0.00 & 0.00 & 0.00 & 0.00 & 0.00\\
&$\sigma =0.50$ & 0.65 & 0.55 & 0.41 & 0.32 & 0.23 & 0.15 & 0.09 & 0.05 & 0.00 & 0.00\\
&$\sigma =1.00$ & 0.47 & 0.39 & 0.34 & 0.28 & 0.22 & 0.17 & 0.14 & 0.12 & 0.10 & 0.08\\
\midrule
\midrule
&$\sigma =0.12$\quad \quad $\epsilon = 0.25$ & 0.82 & 0.69 & 0.00 & 0.00 & 0.00 & 0.00 & 0.00 & 0.00 & 0.00 & 0.00\\
&$\sigma =0.12$\quad \quad $\epsilon = 0.50$ & 0.79 & 0.67 & 0.00 & 0.00 & 0.00 & 0.00 & 0.00 & 0.00 & 0.00 & 0.00\\
&$\sigma =0.12$\quad \quad $\epsilon = 1.00$ & 0.71 & 0.64 & 0.00 & 0.00 & 0.00 & 0.00 & 0.00 & 0.00 & 0.00 & 0.00\\
\bottomrule
\end{tabular}
\end{sc}
\end{tiny}
\end{center}
\end{table}

\begin{table}[h]
\caption{\smoothAdvPGD{} training $T=10$ steps, $m_{train} \in\{1,2,4,8\}$ samples.}
\label{table:cifar_certify_results_PGD_10steps_2/4/8samples}
\begin{center}
\begin{tiny}
\begin{sc}
\begin{tabular}{c c | c c c c c c c c c c}
\toprule
&$\ell_2$ Radius (CIFAR-10)& $0.0$& $0.25$& $0.5$& $0.75$& $1.0$& $1.25$& $1.5$& $1.75$& $2.0$& $2.25$\\
\midrule
\multirow{3}{*}{\rotatebox{0}{\cite{cohen2019certified}}}
&$\sigma =0.12$ & 0.81 & 0.59 & 0.00 & 0.00 & 0.00 & 0.00 & 0.00 & 0.00 & 0.00 & 0.00\\
&$\sigma =0.25$ & 0.75 & 0.60 & 0.43 & 0.27 & 0.00 & 0.00 & 0.00 & 0.00 & 0.00 & 0.00\\
&$\sigma =0.50$ & 0.65 & 0.55 & 0.41 & 0.32 & 0.23 & 0.15 & 0.09 & 0.05 & 0.00 & 0.00\\
&$\sigma =1.00$ & 0.47 & 0.39 & 0.34 & 0.28 & 0.22 & 0.17 & 0.14 & 0.12 & 0.10 & 0.08\\
\midrule
\midrule
&$\sigma =0.12$\quad \quad $\epsilon = 0.25$ & 0.82 & 0.69 & 0.00 & 0.00 & 0.00 & 0.00 & 0.00 & 0.00 & 0.00 & 0.00\\
&$\sigma =0.12$\quad \quad $\epsilon = 0.50$ & 0.79 & 0.68 & 0.00 & 0.00 & 0.00 & 0.00 & 0.00 & 0.00 & 0.00 & 0.00\\
&$\sigma =0.12$\quad \quad $\epsilon = 1.00$ & 0.72 & 0.62 & 0.00 & 0.00 & 0.00 & 0.00 & 0.00 & 0.00 & 0.00 & 0.00\\
\bottomrule
\end{tabular}
\end{sc}
\end{tiny}
\end{center}
\end{table}

\clearpage

\begin{landscape}

\begin{table}[t]
\caption{Certified top-1 accuracy of our best ImageNet classifiers (over a subsample of 500 points)\\ at various $\ell_2$ radii. Standard accuracies are in parantheses.}
\label{imagenet-certified-accuracy-with-clean-accuracy}
\begin{center}
\begin{Large}
\begin{sc}
\begin{tabular}{l| c c c c c c c}
\toprule
$\ell_2$ Radius (ImageNet)& 0.5 & 1.0 & 1.5 & 2.0 & 2.5 & 3.0 & 3.5\\ 
\midrule
\citet{cohen2019certified} (\%) & $^{(67)}$49 & $^{(57)}$37 & $^{(57)}$29 & $^{(44)}$19 & $^{(44)}$15 & $^{(44)}$12 & $^{(44)}$9\\
Ours (\%)& $^{(65)}$\textbf{56} & $^{(55)}$\textbf{45} & $^{(55)}$\textbf{38} & $^{(42)}$\textbf{28} & $^{(42)}$\textbf{26} & $^{(42)}$\textbf{20} & $^{(35)}$\textbf{17}\\
\bottomrule
\end{tabular}
\end{sc}
\end{Large}
\end{center}
\end{table}

\begin{table}[t]
\caption{Certified top-1 accuracy of our best CIFAR-10 classifiers (on the full test set) at various $\ell_2$ radii. Standard accuracies are in parantheses.}
\label{cifar-certified-accuracy-with-clean-accuracy}
\begin{center}
\begin{Large}
\begin{sc}
\begin{tabular}{l | c c c c c c c c c}
\toprule
$\ell_2$ Radius (CIFAR-10)& $0.25$& $0.5$& $0.75$& $1.0$& $1.25$& $1.5$& $1.75$& $2.0$& $2.25$\\
\midrule
\citet{cohen2019certified} (\%)   & $^{(77)}$61 & $^{(66)}$43 & $^{(66)}$32 & $^{(47)}$22 & $^{(47)}$17 & $^{(47)}$13 & $^{(47)}$10 & $^{(47)}$7 & $^{(47)}$4\\
Ours (\%)& $^{(85)}$73 & $^{(76)}$58 & $^{(75)}$48 & $^{(57)}$38 & $^{(53)}$33 & $^{(53)}$29 & $^{(53)}$24 & $^{(44)}$18 & $^{(44)}$16\\
+ Pre-training (\%) & $^{(90)}$80 & $^{(81)}$62 & $^{(74)}$\textbf{52} & $^{(54)}$38 & $^{(54)}$\textbf{34} & $^{(54)}$\textbf{30} & $^{(54)}$\textbf{25} & $^{(41)}$\textbf{19} & $^{(41)}$16\\
+ Semi-supervision (\%) & $^{(90)}$80 & $^{(80)}$\textbf{63} & $^{(80)}$\textbf{52} & $^{(62)}$\textbf{40} & $^{(62)}$\textbf{34} & $^{(52)}$29 & $^{(52)}$\textbf{25} & $^{(42)}$\textbf{19} & $^{(42)}$\textbf{17}\\
+ Both(\%)  & $^{(91)}$\textbf{81} & $^{(81)}$\textbf{63} & $^{(72)}$\textbf{52} & $^{(51)}$37 & $^{(51)}$33 & $^{(51)}$29 & $^{(51)}$\textbf{25} & $^{(41)}$18 & $^{(41)}$16\\
\bottomrule
\end{tabular}
\end{sc}
\end{Large}
\end{center}
\end{table}

\end{landscape}

\end{document}